\documentclass[11pt]{article}
\usepackage[margin=1in]{geometry}
\pdfoutput=1
\usepackage{ifthen}
\newboolean{usenatbib}

\setboolean{usenatbib}{false}

\usepackage[style=alphabetic,maxalphanames=4,minalphanames=3,maxcitenames=2,maxbibnames=99,giveninits=false,doi=false,url=false,natbib]{biblatex}
\DeclareDelimFormat[bib]{nametitledelim}{\addspace}

\usepackage{bm}
\usepackage{fullpage}
\usepackage[english]{babel}
\usepackage[utf8]{inputenc}
\usepackage[T1]{fontenc}
\usepackage{ae} \usepackage{aecompl}
\usepackage{enumitem}
\setlist{nosep}
\usepackage{microtype}
\usepackage{booktabs}
\usepackage{hhline}
\usepackage{makecell}
\usepackage{etoolbox}
\apptocmd{\sloppy}{\hbadness 10000\relax}{}{}

\usepackage{amsmath,amsthm,amssymb,mathtools,thmtools}
\usepackage{graphicx}
\usepackage{color}
\usepackage[dvipsnames]{xcolor}
\usepackage[textsize=scriptsize,disable]{todonotes}
\usepackage[colorlinks=true, allcolors=magneta]{hyperref}
\usepackage[algo2e,ruled]{algorithm2e}
\usepackage{xfrac}
\usepackage[normalem]{ulem}

\hypersetup{colorlinks,
            breaklinks=true,
            linkcolor=purple,
            citecolor=purple,
            urlcolor=magenta,
            linktocpage,
            plainpages=false,
            bookmarks=false}

\usepackage{macros}

\usepackage{comment}
\usepackage{tikz,pgfplots,pgfplotstable}
\usepgfplotslibrary{units}
\usepackage{tabularx} 
\usepackage{subfig}
\usepackage{array}
\usepgfplotslibrary{groupplots}

\title{A Novel Data-Dependent Learning Paradigm for Large Hypothesis Classes}

\author{
    Alireza F. Pour\thanks{University of Waterloo, \texttt{alireza.fathollahpour@uwaterloo.ca}.}
    \and
   Shai Ben-David\thanks{University of Waterloo, \texttt{shai@uwaterloo.ca}.}
}

\numberwithin{equation}{section}

\begin{document}
\maketitle

\begin{abstract}%
  We address the general task of learning with a set of candidate models that is too large to have a uniform convergence of empirical estimates to true losses. While the common approach to such challenges is SRM (or regularization) based learning algorithms, we propose a novel learning paradigm that relies on stronger incorporation of empirical data and requires less algorithmic decisions to be based on prior assumptions. 
  We analyze the generalization capabilities of our approach and demonstrate its merits in several common learning assumptions, including similarity of close points, clustering of the domain into highly label-homogeneous regions, Lipschitzness assumptions of the labeling rule, and contrastive learning assumptions.
  Our approach allows utilizing such assumptions without the need to know their true parameters a priori.
\end{abstract}

\section{Introduction}

We propose a new learning paradigm that closely follows the objective of non-uniform learning, namely, given a set $\cC$ of candidate hypotheses , how many samples do we need to compete with the generalization error of some $h \in \cC$? The canonical approach to non-uniform learning is the Structural Risk Minimization (SRM) algorithm, where the learner structures the hypotheses in a sequence of nested classes. The generalization error that SRM guarantees for any $h\in \cC$ then scales with the index of the first class in which $h$ appears, which is implicit in some weighting function $w(.)$, and the complexity of that class, e.g., its VC dimension. Hence, the success of SRM depends on how well the structure reflects the prior belief about low approximating hypotheses, and how justified that belief is. That is, whether data matching hypotheses are given high weight and/or the class in which they appear is `simple' to learn. For instance, one may assume that the labelling rule is Lipschitz and assign to hypotheses weights that scale with the margin around their decision boundary. The caveat is that the SRM paradigm requires fixing an a priori weighting function and selects its decision rule based on this weighting.

We introduce a different learning paradigm (for classes that are too big for enjoying uniform convergence). We introduce a novel parameter of a collection of hypothesis classes: Given a collection $\bH$, the growth function of the collection, denoted by $\tau_{\bH}(m)$, controls the number of equivalence classes of behaviours that the collection induces on a sample of size $m$. We show that, whenever we group $\cC$ into a collection of hypothesis classes $\bH$, for any $h\in \cC$, it is possible to guarantee a generalization error that scales with $\tau_{\bH}(m)$ and the complexity of the class to which $h$ belongs. In contrast to a weighting function $w(h)$, this parameter is uniform over any $h$ - there is no `penalty' to placing a data fitting $h$ in a low weight class. Thus, we shift the focus to finding a grouping of $\cC$ that minimizes the growth function of the collection of classes and the complexity of the classes that contain low approximating hypotheses. 

This paradigm can be directly applied to partial concepts. A partial concept is a function that can be undefined on parts of the domain with the goal of incorporating data assumptions \citep{alon2022theory}. Given an assumption (or prior knowledge) about the labelling function, a total function $h$ can be transformed into a set of behaviours such that each behaviour aligns with the assumption on its $\{0,1\}$ part, namely, its support. Our paradigm leads the following approach: given a collection $\cC$ of (total) concepts, find a parameter of pairs (a partial function, a finite sample) with the property that (i) it reflects the prior knowledge to enable evaluating to what extent a transformation of a concept into a partial one aligns with the assumption; and (ii) grouping the partial concepts into classes based on this parameter will result in a collection with bounded growth rate. We will show in Section~\ref{sec:applications} that this covers several common choices such as similarity of close points, clustering of the domain, Lipschitzness assumptions on the labeling rule, and contrastive assumptions.

Another difference between our paradigm and SRM is that our proposed paradigm guarantees its error by relying on a compression-based argument. The learner goes over different subsets of the training data, on which it invokes learners $\cA_{\cH}, \cH \in \bH$ that are specifically designed to learn from the classes $\cH \in \bH$. It then estimates the expected error of these predictors on the left out part of the training sample. The learner guarantees the convergence of the error of $\cA_{\cH}$ on the left out set to its expected error with a failure probability that scales with the size of the subset that is used as input training data to $\cA_{\cH}$. Notably, this is one reason that the proposed learner has data-dependent guarantees. Whenever the function belongs to a class with a small data-dependent complexity, we show that the learner needs a smaller subset as training input to learn from that class and thus it will be assigned a large weight when estimating its error on the left out part. We will discuss in Section~\ref{sec:comparison_SRM} the comparison with SRM in more detail. 

\section{Learning with Respect to a Collection of Concept Classes}\label{sec:learning_family_concepts}
In this section, we discuss learnability results with respect to collections of concept classes. 

For a partial concept class $\cH \subseteq \{0,1,\star\}^{\cX}$ and any set $U \in \cX^*$, we define $\cH_{|U} = \{h_{|U}: h\in \cH, h_{|U} \in \{0,1\}^{|U|}\}$. Note that based on this definition $\cH_{|U}$ will always be a subset of $\{0,1\}^{|U|}$ and we never include in $\cH_{|U}$ the behaviours that contain $\star$. For a labelled set $S \in (\cX \times \{0,1\})^*$, we denote by $\dom(S)$ the unlabelled part of $S$. For a set  $S$, we will write $\vc(\cH,S)$ as a shorthand notation for the VC dimension $\vc(\cH_{|{\dom(S)}})$. For any set $S$ and any hypothesis $h: \cX\rightarrow \{0,1,\star\}$ we will denote $\err(h,S) = \frac{1}{|S|} \sum_{(x,y)\in S} 1\{h(x) \neq y\}$. For a distribution $\cD$ over $\cX \times \{0,1\}$ and any hypothesis $h: \cX\rightarrow \{0,1,\star\}$ we will define $\err(h,\cD) = \probs{(x,y)\sim \cD}{h(x) \neq y}$. Moreover, for any hypothesis class $\cH$, we define $\err(\cH,\cD) = \inf_{h \in \cH}(\err(h,\cD))$.

We now give the definitions of the equivalence relation between hypothesis classes and the growth parameter of a collection.
\begin{definition}
    Let $\cH,\cH' \subseteq \{0,1,\star\}^{\cX}$ be two (partial) concept classes defined on $\cX$. Given a set $U \in \cX ^*$, we say $\cH$ and $\cH'$ are equivalent on $S$ if $\cH_{|T} = \cH'_{|T}$ for all $T\subseteq U$ and we denote it by $\cH \sim_U \cH'$. For a family of partial concept classes $\bH$ and for any hypothesis class $\cH \in \bH$ we denote $[\cH_{\sim_U}]_{\bH} = \{\cH' \in \bH: \cH \sim_U \cH'\}$. Moreover, for any set $U\in \cX^*$, we denote the set of all equivalence classes generated by members of $\bH$ on $U$ by $\bH_{|U}:= \{[\cH_{\sim_U}]_{\bH}: \cH\in \bH\}$. 
\end{definition}

\begin{definition}[Growth parameter $\tau_{\bH}$]
    Let $\bH$ be a family of (partial) concept classes. Given a set $U\in\cX^*$ we define $\tau_{\bH}(U) = |\bH_{|U}|$. We denote by $\tau_{\bH}(m)$ the maximum number of the equivalence classes generated on sets of size $m$. Formally,
\(
\tau_{\bH}(m) = \max_{U\in \cX^m} \tau_{\bH}(U) = \max_{U\in \cX^m} |\bH_{|U}|.
\) Moreover, for a distribution $\cD$ and $\delta\in(0,1)$, we denote by $\tau_{\bH}(m,\cD,\delta)$ the smallest integer such that \(\probs{S \sim \cD^m}{\tau_{\bH}(\dom(S)) \leq \tau_{\bH}(m,\cD,\delta)} \geq 1-\delta\).
    \end{definition}
We will sometimes overload the notation and write $\cH_{|S}$ (or $\bH_{|S}$) for labeled samples $S\in (\cX\times \{0,1\})^*$ instead of 
$\cH_{|\dom(S)}$ (or $\bH_{|\dom(S)}$) and write $\tau_{\bH}(S)$ instead of $\tau_{\bH}(\dom(S))$.

We now state our main result for learning with respect to a collection of hypothesis classes.
\begin{theorem}\label{thm:learning_family_concepts}
    Let $\bH$ be a collection of concept classes. There exists a learner $\cA_{\bH}: (\cX \times \{0,1\})^* \times \cX \rightarrow \{0,1\}$, with the following property: for every distribution $\cD$, every 
    $\delta \in (0,1)$ and $m \in \bN$ 
    we have with probability at least $1-\delta$ over samples $S \sim \cD^m$ that
    \[ \begin{aligned}
    &\err(\cA_{\bH}(S),\cD) \leq \min_{\cH \in 
    \bH}\left\{\err(\cH,\cD) + O\left(\sqrt{\frac{\left(\vc(\cH) +  \log(\tau_{\bH}(2m))\right)\log^2 (m) + \log (\frac{1}{\delta})}{m}}\right)\right\}.
   \end{aligned} \]
\end{theorem}

\begin{corollary}\label{cor:learning_union_family}
    Let $\cC$ be a hypothesis class and $\bH$ be a grouping of $\cC$ such that $\bigcup_{\cH \in \bH}\cH = \cC$. There exists a learner $\cA_{\bH}: (\cX \times \{0,1\})^* \times \cX \rightarrow \{0,1\}$, with the following property: for every distribution $\cD$, every 
    $\delta \in (0,1)$ and $m \in \bN$ 
    we have with probability at least $1-\delta$ over samples $S \sim \cD^m$ that
    \[ \begin{aligned}
    &\err(\cA_{\bH}(S),\cD) \leq \min_{h \in 
    C}\left\{\err(h,\cD) + O\left(\sqrt{\frac{\left(\vc(\cH_{(h)}) +  \log(\tau_{\bH}(2m))\right)\log^2 (m) + \log (\frac{1}{\delta})}{m}}\right)\right\},
   \end{aligned} \]
   where $\cH_{(h)}$ is the class to which $h$ belongs.
\end{corollary}
The learner $\cA_{\bH}$ in the above works by going through subsets of the training data. On each subset it invokes a set of learners $\{\cA_{\cH}: \cH \in \bH\}$, estimates their error on the left-out subset, and returns the one that minimizes a bound that depends on the error on left-out subset and the size of the subset input to $\cA_{\cH}$. The learners $\cA_{\cH}$ consider a majority vote of $O(\log(m))$ many One-Inclusion Graph (OIG) learners. Each $\cA_{\cH}$ is based on one-inclusion graph learners 
that predict by only considering the $\{0,1\}$ behaviours of $\cH$ on the input training set and the test point. A boosting technique by considering the OIGs as weak learners guarantees that for a given class $\cH$, there exists a subset of size relevant to its effective VC dimension such that the majority vote of the OIGs on that subset approximately recovers the label of the rest of the sample, i.e., we have a (approximate) compression set for $\cH$ of size proportional to effective VC.

We then prove that the expected error of the learners $\{\cA_\cH: \cH \in \bH\}$ are close to their error on the left-out sample. This is where the introduced growth parameter comes into the play to show that on any subset, there are only a few different predictions that the learners $\cA_{\cH}$ induce on the left out subset. Crucially, the data-dependent benefit we get is due to the fact that we prove the convergence of the error on the sample to the expected error with a probability that depends on the size of the compression set. Moreover, the compression set size depends on the effective complexity of $\cH$, which is unknown in advance. Therefore, contrary to a fixed weighting, the classes are assigned weights based on the compression set we are using to learn from them; there is no mismatch between how complex (with respect to the data distribution) it is to learn a class and how much it is preferred over other classes.

Indeed, we show that it is possible to get the full flavour of a data-dependent bound. The following shows that the same learner in Theorem~\ref{thm:learning_family_concepts}
is capable of guaranteeing bounds that scale with the growth parameter of the grouping, the complexity and the error of the concept classes with respect to the data generating distribution.
For any distribution $\cD$, we denote by $\vc(\cH,\cD,m,\delta)$ the smallest integer such that the probability over $S\sim \cD^m$ that ${\vc(\cH,S) \leq \vc(\cH,\cD,m,\delta)}$ is at least $1-\delta$.
\begin{theorem}\label{thm:family_concepts_data}
    Let $\bH$ be a collection of concept classes. There exists a learner $\cA_{\bH}: (\cX \times \{0,1\})^* \times \cX \rightarrow \{0,1\}$, with the following property: for every distribution $\cD$, every 
    $\delta \in (0,1)$ and $m \in \bN$ 
    we have with probability at least $1-\delta$ over samples $S \sim \cD^m$ that
    \[ \begin{aligned}
    &\err(\cA_{\bH}(S),\cD) \\
    &\leq \min_{\cH \in 
    \bH}\left\{\err(\cH,\cD) + O\left(\sqrt{\frac{\left(\vc(\cH,\cD,m,\frac{\delta}{4}) +  \log(\tau_{\bH}(m,\cD,\frac{\delta}{4}))\right)\log^2 (m) + \log (\frac{1}{\delta})}{m}}\right)\right\}.
   \end{aligned} \]
\end{theorem}

\subsection{Comparison with Structural Risk Minimization (SRM)}\label{sec:comparison_SRM} 

When there is a large set of candidate hypotheses, it may seem natural to apply the structural risk minimization (SRM) principle. However, when grouping hypotheses into a collection of partial concept classes, the canonical SRM may fail, as uniform convergence no longer necessarily holds \citep{alon2022theory}. Consequently, it may be impossible to guarantee— even non-uniformly—that empirical and expected errors converge for every concept in every class. Nevertheless, it is possible to ensure non-uniform convergence of learners tailored to each partial concept class and to construct SRM-type methods that handle such families; see Section~4.2 in \citet{alon2022theory}. Yet, this approach introduces several drawbacks.

SRM requires an a priori weighting over the hypothesis classes, determined independently of the training data. The generalization bound it guarantees for any concept $h$ depends on both the weight $w(\cH_{(h)})$ and the complexity of its class $\cH_{(h)}$. In contrast, the guarantees in Theorems~\ref{thm:learning_family_concepts} and~\ref{thm:family_concepts_data} depend only on the (data-dependent) complexity of $\cH_{(h)}$, while the term $\tau_{\bH}(m)$ (or $\tau_{\bH}(m,\cD,\delta)$) is uniform across all $h$. Consequently, SRM may fail to provide a reasonable bound when the well-approximating hypothesis $h^*$ is assigned a small weight.
SRM selects its predictor by minimizing an upper bound on the error, derived from non-uniformly requiring uniform convergence for each hypothesis class.\footnote{For comparison with SRM, we assume that the collection contains total classes mapping the domain to ${0,1}$.} Such worst-case guarantees scale with the (worst-case) VC dimension of each class and ignore the possibility that the effective complexity on the underlying distribution may be much smaller. For instance, a class $\cH_{(h)}$ may have infinite VC dimension but finite $\vc(\cH_{(h)}, \cD, m, \delta)$. To address this, several works propose SRM variants based on sample-dependent complexity measures, such as Rademacher or localized complexities \citep{koltchinskii2002rademacher, bartlett2002model}, which attempt to ensure non-uniform convergence relative to $\vc(\cH_{(h)}, \cD, m, \delta)$ rather than its worst-case counterpart.

The key insight is that uniform convergence must hold (even non-uniformly) across all hypothesis classes. An SRM structures hypotheses into a (possibly nested) collection and assigns each class $\cH_{(h)}$ a weight $w(\cH_{(h)})$ a priori, before observing data. Since it cannot anticipate the distribution-dependent complexity $\vc(\cH_{(h)}, \cD, m, \delta)$, it must rely on the worst-case $\vc(\cH_{(h)})$. 
Indeed, a common approach is to construct a nested hierarchy of hypothesis classes based on how well they satisfy the available prior knowledge: as the index increases, the classes become richer and more complex, while the assigned weights decrease accordingly. 
This introduces an additive term $O(\sqrt{\log (1/w(\cH_{(h)}))/m})$ in the generalization bound SRM considers for returning a predictor, which can undermine the benefits of small effective VC dimension. As a result, even if a high-complexity class achieves near-zero approximation error, SRM is unlikely to select it over a simpler but highly inaccurate class.

The following proposition is an abstract application of the last point above. We call an SRM learner \emph{standard}, if the learner returns a hypothesis $\hat{h}$ that has the minimum value of the generalization upper bound
\[\err(\hat{h},S) + O\left(\sqrt{\frac{\vc(\cH_{(\hat{h})},\cD,m,\delta) +\log(1/w(\cH_{(\hat{h})}))+ \log(1/\delta)}{m}}\right),\]
 which is based on uniform convergence of the hypothesis class with failure probability proportional to the weight of class. The proof appears in Appendix~\ref{appendixB}.
\begin{proposition}\label{proposition:bad_srm}
There exists a collection $\bH = \{\cH_n: n \in \bN\}$ with $\vc(\cH_n)=n$ for all $n\in \bN$ such that for any \emph{standard} SRM learner $\cA_{\srm}$ the following holds. For every $\delta\in(0,1)$ and sample size $m$, there exists a distribution $\cD$ such that \(\probs{S \sim \cD^m} {\err(\cA_{\srm}(S),\cD) = 1/2}\geq 1-\delta\). Moreover, for the learner $\cA$ in Theorem~\ref{thm:learning_family_concepts} we have with probability at least $1-\delta$ over $S\sim \cD^m$ that
 \[ \begin{aligned}
 \err(\cA(S),\cD) 
    \leq O\left(\sqrt{\frac{\log^2 (m) + \log (1/\delta)}{m}}\right).
    \end{aligned}\]
\end{proposition}

Theorems~\ref{thm:learning_family_concepts} and~\ref{thm:family_concepts_data} thus offer generalization bounds comparable to those of the Empirical Risk Minimizer (ERM) over any $\cH \in \bH$, with a uniform additive factor of $O\left(\sqrt{\frac{\log(\tau_{\bH}(m))\log^2(m)}{m}}\right)$. In contrast, the bound for SRM deviates from that of ERM by a non-uniform term depending on $w(\cH)$; see Section~7.2 in \citet{shalev2014understanding}.

We will instantiate our results in settings where a large candidate set $\cC$ and an additional source of prior knowledge are available. The prior knowledge is used to group $\cC$ into a nested family $\bH = \{\cH_r: r\geq 0\}$ with $\cH_r \subseteq \cH_s$ for $r \le s$, where the complexity of classes grow with $r$. A standard SRM assigns larger weights to simpler classes (small $r$) to favor lower estimation error. Consequently, as $\cH_r$ approaches the full class $\cC$, its weight $w(\cH_r)$ decreases, and the term $\sqrt{\frac{\log (1/w(\cH_r))}{m}}$ in the SRM error bound increases—making it less competitive with ERM over $\cC$. In contrast, the learner of Theorems~\ref{thm:learning_family_concepts} and \ref{thm:family_concepts_data} treats all classes uniformly, and thus competes with the ERM bound even when only large $r$ yield small approximation error—comparable to ignoring the additional source altogether.

Finally, when $\bH$ is uncountable, SRM requires an a priori discretization of the family. While for nested sequences it is possible to select breakpoints $\{r_i:i\in \bN\}$ where the VC dimension changes, such discretization still fails to guarantee competitiveness with every $\cH_r$. Between two discretization points, one either loses approximation accuracy (when using $\cH_{r_i}$) or incurs higher complexity (when using $\cH_{r_{i+1}}$). Instead of fixing discretization in advance, our approach lets the data itself to partition $\bH$ into equivalence classes, ensuring that each restricted class ${\cH_r}_{|S}$ is considered as a candidate. This allows the resulting bound to compete directly with $\err(\cH_r,\cD)$ whenever $\tau_{\bH}(m)$ grows polynomially, without assuming prior discretization or explicit weighting.

\section{The Growth Parameter $\tau_{\bH}$}

We can see that while a polynomial bound on the growth of $\tau_{\bH}(m)$ is sufficient to learn from $\bH$, there are collection of classes with exponential growth that are still easy to learn.

For instance, assume the universe is $[m]$. For any subset $A \subset [m]$ of size $\log(m)$ let $h_{A}^{(1)},\ldots,h_{A}^{(m)}$ be functions that shatter the set $A$ such that they all extend to  $[m]\setminus A$ with a label of $1$, i.e., they form a cube on $[m]$ with dimension set $A$. Let $\sigma(1),\ldots,\sigma(2^m)$ be any enumeration of the subsets of the set $[m]$. Now, for any $A\subset [m]$ and $i\in[2^m]$ define a hypothesis class $\cH_{A}^{(i)} = \{h_{A}^{(j)}: j\in \sigma(i)\}$. Let $\bH = \{\cH_{A}^{(i)}: A\subset [m], |A| = \log(m), i\in [2^m]\}$. It is clear that on any set $S$ of $m$ distinct points, we have $\tau_{\bH}(S) = 2^m$. On the other hand, we can simply see that for the union $\cC = \bigcup_{\cH_{A}^{(i)} \in \bH}$, we have $\vc(\cC) \leq \log(m)$ and an ERM is a good learner for the union.
 
 We will now list some behaviours of $\tau_{\bH}(S)$ and then discuss interesting cases in the setting of nested hypothesis classes. We denote by $\pi_{\cC}(S)$ the number of behaviours $\cC$ induces on $S$ and by $\pi_{\cC}(m)$ the maximum number of $\pi_{\cC}(S)$ over all sets $S$ of size $m$. 
\begin{enumerate}
    \item If $\bH = \{\cH\}$ is a singleton then $\tau_{\bH}(S) = 1$ and $\max_{\cH \in \bH} \vc(\cH,S) = \vc(\cC,S)$
    \item If $\bH = \{\{h\}: h \in \cC\}$ then $\tau_{\bH}(S) = \pi_{\cC}(S)$ so $\tau_{\bH}(m) =\pi_{C}(m)$ and $\max_{\cH \in \bH} \vc(\cH,S) = 0$
    \item If $\bH = \{\cH_{r}: r\geq 0\}$ is a nested collection, i.e., $\cH_{r} \subseteq \cH_{s}$ for $r\leq s$ then we have $\tau_{\bH}(S) \leq \pi_{\cC}(S)$. 
    \item $\tau_{\bH}(m)$ is non-decreasing, i.e., for any $m,m'\in\bN$ with $m\leq m'$,$\tau_{\bH}(m) \leq \tau_{\bH}(m')$.
\end{enumerate}
We will show some cases where $\tau_{\bH}(S)$ is much smaller than $\pi_{\cC}(S)$.

\subsection{Hierarchical Clustering}\label{section:HC}

Assume we have a finite domain $\cX$ and $\cH \subseteq\{0,1\}^{\cX}$. We also have, as an additional knowledge, a hierarchical clustering over the domain that we wish to incorporate in learning. We want to assume each level as a possible clustering of the domain that groups the instance into label-homogenous clusters. Therefore, we want to define a class $\cH_i$ for each level that respects the clustering of level $i$. We now discuss formally the hierarchical clustering and the classes $\cH_i$.

A clustering $\cC$ is defined by a weight function $w_{\cC}:\cX \times \cX \rightarrow \{0,1\}$ that is symmetric and transitive in the $1$ value, i.e., $w_{\cC}(x,x') =1$ and $w_{\cC}(x',x'')=1$ implies $w_{\cC}(x,x'')=1$. A clustering partitions the domain into equivalence classes. A cluster $C$ in $\cC$ is any equivalence class, i.e., for all $x \in C$, and for all $x'\in\cX$, $w(x,x') = 1$ implies $x'\in C$. 

A hierarchical clustering is a tree of clusterings such that each node $v$ in this tree is a cluster $C_v$ and all the clusters at depth $i$ create a clustering of the domain, which we denote by $\cC_i$. In the leafs of this graph we have a clustering $\cC_{\cX}$ that has a cluster for every $x \in \cX$, i.e., $w_{\cC_{\cX}}(x,x') = 0$ for all $x,x'\in\cX$. For every other node $v$ in this graph, the cluster $C_v$ merges the clusters of its children, i.e., $C_v = \bigcup_{d\in d(v)} C_d$ where $d(v)$ denotes the children of a node $v$. This means that the hiearachical clustering starts by assigning every point to a cluster. Each time it merges some clusters and moves up the tree until we have at the root the clustering $\cC_0$ that collects every point into a single cluster.

We now define the collection of hypotheses classes as those that incorporate this hierarchical clustering into $\cH$. We want to consider the entire hierarchy in a collection because it is not clear which $\cC_i$ is the best approximation of the data-generating distribution $\cD$ and that we want to make a trade-off between how the clustering `fits' the true labeling and how `complex' it is to learn from that clustering.

The translation of $\cH$ into a partial concept that is enforced to respect the homogeneity over the clustering $\cC$ is defined as 
\[
\begin{aligned}
\cH(\cC):= &\{h \in \{0,1,\star\}^{\cX} : \exists h'\in\cH, \forall x\in \cX, \text{$h(x) = \star$ or $h(x) = h'(x)$} \\
& \text{ and }\forall x_1,x_2 \in \cX \text{ with $h(x_1),h(x_2)\in \{0,1\}$, and $w_{\cC}(x_1,x_2) =1$, $h(x_1) =h(x_2) $} \}.
\end{aligned}\]
We then define $\bH_{\mathsf{HC}} = \{\cH_i\}$ and note that $\cH_i \subseteq \cH_j$ for all $i\leq j$. The interesting observation here is that although $|\bH_{\mathsf{HC}}|$ can be as large as $ |\cX|$ and we have many hypothesis classes, the growth of this collection is only linear in the sample size.
\begin{claim}
    For the hierarchical clustering and the collection $\bH$ defined above, we have for every $S \in \cX^*$ that $\tau_{\bH_{\mathsf{HC}}}(S) \leq |S|$. Hence, $\tau_{\bH_\mathsf{HC}}(m) \leq m$. 
\end{claim}
\begin{proof}
  
   Let $K_i(S)$ denote the number of clusters in $\cC_i$ such that at least a point from $S$ resides in the cluster. Observe that $K_{\cX}(S) \leq |S|$ and that whenever $w_{\cC_i}(x,x') = 0$ and $w_{\cC_{i-1}}(x,x') = 1$ we have merged the two clusters that contain $x$ and $x'$. Therefore, we have $K_{i-1}(S) \leq K_{i}(S) - 1$ and the number of clusters that contain at least a point from $S$ decreases. Moreover, for any $i\leq j$ if $w_{\cC_j}(x,x') = 1$ then $w_{\cC_i}(x,x') = 1$. Also, for any $i \leq j$ such that $w_{\cC_i}(x,x') = w_{\cC_j}(x,x')$ for all $ x,x'\in S$ we have that $\cH_{i|T} = \cH_{j|T}$ for all $T \subseteq S$ and therefore $\cH_i \sim_{S} \cH_{j}$. This means that we can have at most $K_{\cX}(S) \leq |S|$ equivalence classes, concluding the proof.
\end{proof}

\subsection{Forbidden Behaviours on Tuples of Points}

In this section, we study settings where extra information appears as forbidden behaviours on tuples of points. These constraints are graded by a real-valued penalty, indicating how strongly a behaviour is disallowed. 
Limiting the allowed behaviours on the sample could reduce the complexity of learning but can also increase the chance of error. The penalty value reflects this trade-off between approximation and estimation. We define partial concept classes by giving a threshold to the penalty and only allowing behaviours below a chosen value. The goal is to find a threshold of penalty, i.e., a set of allowed behaviours, that minimizes the error bound among all the values for the data-generating distribution.

Formally, let $\cY = \{0,1\}^k$ and $\cB:\cX^k \rightarrow 2^{\cY}$ be a function that gets as input a sample of size $k$ and outputs a set of forbidden behaviours on those $k$ instances. 
Moreover, let $p:\cX^k \rightarrow \bR$ be a penalty function. For any $r\geq 0$, the class $\cH_{\cB,p}(r) \subseteq \{0,1,\star\}^*$ is defined as 
\[
\begin{aligned}
\cH_{\cB,p}(r) = \{h \in \{0,1,\star\}^*: \,&\exists h'\in \cH, \forall x \in \cX \text{ if $h(x) \neq \star$ then $h(x) = h'(x)$ }, \\&\text{ and $\forall S\in \cX^k$ if $p(\dom(S)) \geq r$ then $h_{|S} \notin \cB(S)$}\}
\end{aligned}\]
When the forbidden behaviours $\cB$ are clear from context we often drop. Moreover, we will overload the notation and write $p(S)$ instead of $p(\dom(S))$. We sometimes drop $p$ and simply write $\cH(r)$. 

We then define the family of concept classes for the above assumption by $\bH_{\cB,p} = \{\cH_{\cB,p}(r): r\geq 0\}$ and show that its growth rate is only polynomial in $m$.
\begin{lemma}\label{lemma:growth_forbidden}
For any hypothesis class $\cH$ and the smooth family $\bH_{\cB,p}$ as defined above, we have $\tau_{\bH_{\cB,p}}(m) = O(m^k)$.
\end{lemma}

\begin{proof}
Fix a set $S$ of size $m$ and let $\ell:= {m \choose k}$ . Denote by $A_1,\ldots,A_{\ell} \subseteq S$ all the subsets of $S$ of size $k$. Let $p_1\leq \ldots \leq p_{\ell}$ be the values of $\{p(A_i): i \in [\ell]\}$ sorted in increasing order. Fix any $h \in \{0,1,\star\}^{S}$ such that $\exists h'\in \cH, \forall x \in S$ if $h(x) \neq \star$ then $h(x) = h'(x)$. Then for any subset $T \subset S$ with $|T| <k$ we either have $h(x) \in \{0,1\}$ for all $x\in T$ in which case $h_{|T} \in \cH(r)_{|T}$ for all $r\geq 0$ or $h(x) = \star$ for some $x \in T$ in which case the behaviour is not included in any $\cH(r)_{|T}$. For any $T \subseteq S$ with $|T| \geq k$, let $A(T) \subseteq \{A_1,\ldots,A_{\ell}\}$ be the collection of subsets of size $k$ that are also subsets of $T$, i.e., $A(T) = \{A_i: A_i \subseteq T, i\in[\ell]\}$. Now we either have $h(x) \in \{0,1\}$ for all $x\in T$ or $h(x) = \star$ for some $x\in T$. In the former case, if $h_{|A} \notin \cB(A)$ for all $A \in A(T)$ then $h_{|T} \in \cH(r)_{|T}$ for all $r\geq 0$. Otherwise, $h_{|T} \in \cH(r)_{|T}$ only for $r > \max \{p(A): A \in A(T), h_{|A} \in \cB(A)\}$. Assume $A_j \in A(T)$ is the subset achieving $\max \{p(A): A \in A(T), h_{|A} \in \cB(A)\}$. This implies that $h_{|T} \in \cH(r)_{|T}$ for all $r > p_j$ and $h_{|T} \notin \cH(r)_{|T}$ otherwise. In the case that $h(x) = \star$ for some $x\in T$, the behaviour is not included in any $\cH(r)_{|T}$. This concludes that for all subsets $T \subseteq S$, the behaviours in $\cH(r)_{|T}$ are exactly the behaviours in  $\cH(r')_{|T}$ if $p_j \leq r,r' < p_{j+1}$ for some $j \in [\ell]$ or if both $ r,r'< p_1$ or both $ r,r'\geq p_\ell$. Hence, $\cH(r) \sim_S \cH(r')$ if $p_j \leq r,r' < p_{j+1}$ for some $j \in [\ell]$, if $r,r'< p_1$, and if $r,r \geq p_\ell$. The claim follows.
\end{proof}

The concept of forbidden behaviours can be reminiscent of regularization penalties. We want to highlight that forbidden behaviours penalize violations of constraints deriven from prior knowledge, while regularization typically limits the complexity of the hypothesis class. In certain applications—such as the similarity function discussed earlier—these two notions may coincide conceptually. Yet, the results in this section focus on learning from such restriction by quantifying how much they are violated on the training sample. Notably, it is unclear how these constraints could be incorporated into standard regularization penalties for total concepts. For instance, it is not straightforward to include the Lipschitz-type similarity restriction of a binary total concept into a regularization penalty, since almost any total function on a continuous domain would violate it. In contrast, the forbidden behaviour framework distinguishes hypotheses based on the extent to which they violate the Lipschitz criteria on the finite training set. It offers a tool to learn from such penalties by considering each candidate concept based on how well it classifies the sample and how much it violates the restrictions on it.

The bound on $\tau_{\bH_{\cB_{k},p}}(m)$ in Lemma~\ref{lemma:growth_forbidden} leads the following result for cases where we group the candidate set based on the penalties over forbidden behaviours. 
\begin{theorem}\label{thm:family_forbidden_sample}
    Let $\cH\subseteq \{0,1\}^{\cX}$ be a hypothesis class and $p$, be the penalty function on some forbidden behaviours $\cB$. There exists a learner $\cA$, with the following property: for every distribution $\cD$, every 
    $\delta \in (0,1)$ and $m \in \bN$ 
    we have with probability at least $1-\delta$ over samples $S \sim \cD^m$ that
     \[ \begin{aligned}
    &\err(\cA(S),\cD) \leq \min_{r \geq 0}\left\{\err(\cH(r),\cD) + O\left(\sqrt{\frac{\left(\vc(\cH(r),D) +k\right)\log^2 (m) + \log (1/\delta)}{m}}\right)\right\}.
   \end{aligned} \]
\end{theorem}
{\bf Structural risk minimization over data-dependent hierarchies.} \citet{shawe2002structural} introduce an algorithmic paradigm based on a luckiness function—a measure on data used to prefer hypotheses that appear `luckier' on the training sample. While a complete comparison is out of the scope of this work, we highlight that their framework applies only to the realizable setting, as it relies on a technical assumption called \emph{probably smoothness}, which bounds the number of hypotheses that are as lucky as a hypothesis on a double-sample. This condition enables an error guarantee for an ERM that selects the luckiest hypothesis. In contrast, our results address the general agnostic setting. Moreover, the framework developed here builds on a broader paradigm that allows learning from a collection of partial concepts by structuring into families of bounded growth.

\section{Applications}\label{sec:applications}
In this section we discuss some application of forbidden behaviours. 

\subsection{Similarity Graph of the Domain}\label{section:similarity_graph}
We assume that the prior knowledge comes as a weighted graph on the domain. Specifically, there exists a weight function $w: \cX\times \cX \rightarrow [0,1]$ where $w(x_1,x_2)$ reflects the certainty of the prior knowledge on the similarity between the labels of $x_1$ and $x_2$.

Based on this prior knowledge, we define the following forbidden behaviours and their penalties.
\begin{definition}[$\cB_{w},p_{w}$]
The forbidden behaviours for similarity graph is defined as having opposite labels, i.e., $\cB_{w}(x_1,x_2)=\{(0,1),(1,0)\}$, and its penalty is the weight of the edge between them, i.e., $p_{w}(x_1,x_2) = w(x_1,x_2)$.
\end{definition}
We then define the class of functions $\cH_{p_w}$ as
\[ \begin{aligned}
\cH_{p_w}(r) = \{h &\in \{0,1,\star\}^{\cX} : \exists h'\in\cH, \forall x\in \cX, \text{$h(x) = \star$ or $h(x) = h'(x)$} \\
& \text{ and }\forall x_1,x_2 \in \cX \text{ with $h(x_1),h(x_2)\in \{0,1\}$, and $w(x_1,x_2) \geq r$, $h(x_1) =h(x_2) $} \}.
\end{aligned}\]

For a weight function $w$, we define shattering at weight $r$ by counting behaviours that give opposing labels only to points with weight less than $r$. Similar notions are defined for robust learning and learning with augmentation \citep{montasser2019vc,attias2022characterization,shao2022theory}. 
\begin{definition}[$\vc_{w,r}(\cH)$] Let $\cH\subseteq \{0,1\}^{\cX}$ be a hypothesis class. For any $r \geq 0$, a set $U\in \cX^*$
is called to be shattered at weight $r$ if (i) $U$ is shattered by $\cH$ and (ii) for any $x_1,x_2\in U$ we have $w(x_1,x_2) < r $. The $\vc_{w,r}(\cH)$ dimension of the hypothesis class $\cH$ at distance $r$ is then defined as the size of the largest shattered set by distance $r$.
\end{definition}
We know from Lemma~\ref{lemma:growth_forbidden} that $\tau_{\bH_{\cB_w,p_w}}(m) = O(m^2)$. We can now apply Theorem~\ref{thm:family_forbidden_sample} to conclude the following two bounds that combines the prior knowledge of a hypothesis class with the partial similarity knowledge. 
\begin{theorem}\label{thm:similarity_graph}
    Let $\cH\subseteq \{0,1\}^{\cX}$ be a hypothesis class. There exists a learner $\cA$, with the following property: for every distribution $\cD$, every 
    $\delta \in (0,1)$ and $m \in \bN$ 
    we have with probability at least $1-\delta$ over samples $S \sim \cD^m$ that
     \[ \begin{aligned}
    &\err(\cA(S),\cD) \leq \min_{r \geq 0}\left\{\err(\cH(r),\cD) + O\left(\sqrt{\frac{\left(\vc_{w,r}(\cH)\right)\log^2 (m) + \log (1/\delta)}{m}}\right)\right\}.
   \end{aligned} \]
\end{theorem}

\begin{remark}
    In Section~\ref{section:HC} we defined hierarchical clustering as a way of expressing prior knowledge. We note that such a knowledge can also be given to the learner by giving the criteria of the clustering, i.e., a notion of similarity and the entire domain $\cX$ as the unlabeled data. Each level of the hierarchy then joins clusters with a similarity above some threshold. We showed that in such a setting the growth rate of the clusterings at different levels, i.e., thresholds of the similarity, is $\tau_{\bH_\mathsf{HC}}(m)\leq m$. The setting in this section gives the learner less knowledge—only the similarity measure, with no extra unlabeled data—and the resulting collection has a growth rate of $\tau_{\bH_{\cB_w,p_w}}(m)=O(m^2)$. As discussed earlier, the guarantee in Theorem~\ref{thm:similarity_graph} is achieved by taking multiple OIG learners trained on different subsets of the labeled sample, aggregating them by majority vote, and selecting the predictor with the tightest bound on the held-out.
   Changing the training subset changes the induced clustering the OIGs `see', thereby increasing the number of equivalence classes. Contrary to Section~\ref{section:HC}, the knowledge of the clusterings on the entire sample is not given to them a priori. An interesting question would be understanding how we can improve over the bound in Theorem~\ref{thm:similarity_graph} by only having access to a finite unlabeled sample.
\end{remark}

\subsection{Incorporating Prior Knowledge with Nearest Neighbour}
In the following, we derive generalization bounds that take into account both the assumptions of a hypothesis class and also the assumption that labels of close points should be similar.
\begin{definition}[$\cB_{\nn},p_{\nn,\phi}$]
The forbidden behaviours for nearest neighbour is defined as having opposite labels, i.e., $\cB_{\nn}(x_1,x_2)=\{(0,1),(1,0)\}$, and its penalty is the inverse of the distance function $\phi :\cX^2 \rightarrow \bR$ between the points, i.e., $p_{\nn,\phi}(x_1,x_2) = 1/\phi(x_1,x_2)$.
\end{definition}

The class of functions $\cH_{p_{\nn,\phi}}(1/r)$ for $r\geq 0$ based on the above forbidden behaviours will then become 
  \[\begin{aligned} \cH_{p_{\nn,\phi}}(1/r) = &\{h \in \{0,1,\star\}^{\cX} : \exists h'\in\cH, \forall x\in \cX, \text{$h(x) = \star$ or $h(x) = h'(x)$} \\
& \text{ and }\forall x_1,x_2 \in \cX \text{ with $h(x_1),h(x_2)\in \{0,1\}$, and $\phi(x_1,x_2) \leq r$, $h(x_1) =h(x_2) $} \}.\end{aligned}\] 

We define shattering at distance $r$ by only allowing behaviours that do not give opposing labels to points with distance smaller than $r$. This will then lead to the notion of VC dimension of $\cH_{p_{\nn,\phi}}$.
\begin{definition}[$\vc_{\phi,r}(\cH)$] Let $\cH\subseteq \{0,1\}^{\cX}$ be a hypothesis class. For any $r \geq 0$, a set $U\in \cX^*$
is called to be shattered by distance $r$ if (i) $U$ is shattered by $\cH$ and (ii) for any $x_1,x_2\in U$ we have $\phi(x_1,x_2) > r $. The $\vc_{\phi,r}(\cH)$ dimension of the hypothesis class $\cH$ at distance $r$ is then defined as the size of the largest shattered set by distance $r$.
\end{definition}
Noting that $\cB_{\nn}$ is a forbidden behaviour on pairs of points so $\tau_{\bH_{\cB_{\nn},\phi}}(m) = O(m^2)$, and that $\vc(\cH(1/r)) = \vc_{\phi,r}(\cH)$ we can apply Theorem~\ref{thm:family_forbidden_sample} to conclude the following bound. This bound combines the prior knowledge of a hypothesis class with the smoothness of the labeling function, without requiring the knowledge of the optimum parameter of smoothness.

\begin{theorem}\label{thm:similarity_distance}
    Let $\cH\subseteq \{0,1\}^{\cX}$ be a hypothesis class. There exists a learner $\cA$, with the following property: for every distribution $\cD$, every 
    $\delta \in (0,1)$ and $m \in \bN$ 
    we have with probability at least $1-\delta$ over samples $S \sim \cD^m$ that
     \[ \begin{aligned}
    &\err(\cA(S),\cD) \leq \min_{r \geq 0}\left\{\err(\cH(1/r),\cD) + O\left(\sqrt{\frac{\left(\vc_{\phi,r}(\cH)\right)\log^2 (m) + \log (1/\delta)}{m}}\right)\right\}.
   \end{aligned} \]
\end{theorem}

\subsubsection{Nearest neighbour with no prior knowledge about a concept class.}
We now show that if we set $\cH$ to be the class of all functions, then we would recover bounds similar to some of the bounds derived for nearest neighbour in literature \citep{gottlieb2014near,kontorovich2017nearest,hanneke2021stable}, etc.
\begin{proposition}
    Let $\cH$ be the class of all functions from $\cX$ to $\{0,1\}$. Then, $\vc_{\phi,r}(\cH) = |N_r|$, where $N_r$ is the $r$-net of the domain.
\end{proposition}
\begin{proof}
    Consider any maximal set $U$ that is shattered. Any $x,x'\in U$ must have $\phi(x,x')>r$. Therefore, $U$ is a $r$-packing of $\cX$. Assume for the sake of contradiction that $U$ is not a $r$-cover of $\cX$. Then there exists $x_u \in \cX$ such that $x_u \notin B(x,r)$ for all $x\in U$, where $B(x,r)$ is the ball of radius $r$ centered at $x$. But this implies that $\phi(x_u,x) > r$ for all $x\in U$. Therefore, $U \bigcup x_u$ is also shattered contradicting the maximality of $U$. This concludes that $U$ must be an $r$-net of $\cX$. 
\end{proof}
\begin{theorem}
    Let $\cH$ be the class of all functions from $\cX$ to $\{0,1\}$. Assume $\diam(\cX) \leq R$ and $\ddim(\cX) \leq d$, where $\dim(\cX)$ and $\ddim(\cX)$ are the diameter and doubling dimension of the space. There exists a learner $\cA: (\cX \times \{0,1\})^* \times \cX \rightarrow \{0,1\}$, with the following property: for every distribution $\cD$, every 
    $\delta \in (0,1)$ and $m \in \bN$ 
     we have with probability at least $1-\delta$ over samples $S \sim \cD^m$ that
       \[ \begin{aligned}
 \err(\cA(S),\cD) 
   & \leq \min_{r\geq 0 } \left\{\err(\cH(1/r),\cD) 
    + O\left(\sqrt{\frac{|N_{r}| \log^2 (m) + \log (1/\delta)}{m}}\right)\right\} \\
    &\leq \min_{r\geq 0 } \left\{\err(\cH(1/r),\cD) 
    + O\left(\sqrt{\frac{\left((R/r)^{d+1} \right) \log^2 (m) + \log (1/\delta)}{m}}\right)\right\}.
    \end{aligned}\]
    \end{theorem}

\subsection{Learning from Real-Valued Functions}   
In this section we show how the introduced paradigm can be used to learn from real-valued functions. We also show how the generalization bound achieved through this parameter compares with the well-known bounds on the fat-shattering dimension. We first define the necessary components to understand such bounds.

For a class $\cF$ of real-valued functions from $\cX$ to $\bR$, we define by $\cF(L)\subseteq \cF$ the set of all functions in $\cF$ with Lipschitz constant at most $L$. We also define by $\cF^{0/1} = \{f^{0/1}\in \{0,1\}^{\cX}: \exists f \in \cF,\forall x\in\cX, f^{0/1}(x) = \sign(f(x))\}$. For a real-valued function $f$ define $\err^{\gamma}(f,\cD) = \probs{(x,y)\sim \cD}{yf(x) < \gamma}$ and for any set 
    $S$ define $\err^{\gamma}(f,S) = 1/|S| \sum_{(x,y)\in S} 1\{yf(x) < \gamma\}$. Moreover, define $\err^{\gamma}(\cF,\cD) = \inf_{f\in\cF} \err^{\gamma}(f,\cD)$ and $\err^{\gamma}(\cF,S) = \inf_{f\in\cF} \err^{\gamma}(f,S)$.

We further define  $\err^{*}(\cH,\cD,n) = \expects{S \sim \cD^n} {\min_{h \in \cH} \err(\cH,S)}$ as a related error rate for a partial concept $\cH$. This error is defined in \citet{alon2022theory} to set the benchmark of PAC learning partial concepts as competing with $\err^{*}(\cH,\cD) = \lim_{n\rightarrow \infty} \err^{*}(\cH,\cD,n)$. All the generalization guarantees we have also hold for the benchmark of $\err^{*}(\cH,\cD)$. Moreover, it is shown in \citet{alon2022theory} that we always have $\err^{*}(\cH,\cD) \leq \err(\cH,\cD)$.

We now show how the $\err^{\gamma}$ of a class of real-valued functions is compared with the $0/1$ error of partial concepts $\cF^{0/1}(1/r)$ that respect similarity of points at different distance levels $r$. Our bounds will be based on the error of $\cF^{0/1}(.)$ and its VC dimension.
\begin{proposition}
Let $\cF$ be a class of real-valued functions from $\cX$ to $\bR$. For any sample $S \in (\cX \times \{0,1\})^*$, we have that $\err(\cF^{0/1}(L/2\gamma),S) \leq \err^{\gamma}(\cF(L),S)$. Furthermore, for any distribution $\cD$ and any $L >0$, we have $\err^{*}(\cF^{0/1}(L/2\gamma),\cD) \leq \err^{\gamma}(\cF(L),\cD)$.
\end{proposition}
\begin{proof}
    Let $f_L\in \cF(L)$ and $S=\{(x_i,y_i)\}_{i=1}^n$ be any set. Define $b_i = \sign(f_L(x_i))$ if $|f_L(x)|\geq \gamma$ and $b_i=\star$ otherwise. We prove that the function $f_S$ with $f_S(x_i) = b_i$ for $i\in [n]$ and $f_S(x) = \star$ for $x \notin \dom(S)$ is in $\cF^{0/1}(L/2\gamma)$. If $b_i =1$ and $b_j=0$ for some $i,j\in[n]$ then we have $f(x_i) \geq \gamma$ and $f(x_j) \leq -\gamma$. This implies that $|f(x_i) - f(x_j)| \geq 2\gamma$ and therefore, $|x_i-x_j| \geq 2\gamma/L$. This concludes that $f_S \in \cF^{0/1}(L/2\gamma)$. On the other hand, for any $b_i = \star$ we have $y_if_L(x_i) < \gamma$. Therefore, $f_L$ $\gamma$-errs on $x_i$ and so does $f_S$. For any $b_i \neq \star$, $f_L$ will $\gamma$-err on $x_i$ if and only if $\sign(f_L(x_i)) \neq y_i$, which means that $f_S$ will err on $x_i$ because $f_S(x_i) = \sign(f_L(x_i))$. The above concludes that for all $i\in [n]$, $f_S$ will error on $(x_i,y_i)$ if and only if $f_L$ $\gamma$-errs on $(x_i,y_i)$. We can therefore conclude that $\err(f_S,S) = \err^{\gamma}(f_L,S)$. Since we proved for any $f_L\in \cF(L)$ such $f_S$ exists, we get that $\err(\cF^{0/1}(L/2\gamma),S) \leq  \err^{\gamma}(\cF(L),S)$. 
    
    Moreover, the above also concludes that for any $f_L$ and $n$ we have that \[\err^{\gamma}(f_L,\cD) = \expects{S \sim \cD^n}{\err^{\gamma}(f_L,S)} \geq \expects{S \sim \cD^n}{\min_{f \in \cF^{0/1}(L/2\gamma)}\err(f,S)} = \err^{*}(\cF^{0/1}(L/2\gamma),\cD,n).\]
Since this holds for any $f_L\in\cF(L)$ and any $n$, by letting $n\rightarrow \infty$ we get that $\err^{*}(\cF^{0/1}(L/2\gamma),\cD) \leq \err^{\gamma}(\cF(L),\cD)$.
\end{proof}
We give the definition of fat-shattering dimension and then compare it with the VC dimension of the translated partial functions $\cF^{0/1}(.)$. This definition is introduced in \citet{kearns1990efficient} as a version of the pseudo-dimension \citep{pollard1990empirical} that scales by a margin parameter.
\begin{definition}
    We say a set $U\in\cX^*$ is $\gamma$-shattered by a class $\cF$ of functions from $\cX$ to $\bR$ if there exists a function $r:\cX\rightarrow \bR$ such that for every $b \in \{-1,1\}^{U}$, there exists a function $f\in\cF$ with $b(x)(f(x)-r(x))\geq \gamma$ for all $x\in U$. The fat shattering dimension of the class $\cF$ at scale $\gamma$ is denoted by $\fatsh_{\gamma}(\cF)$ and is the maximum size of a set that is $\gamma$-shattered by $\cF$.
\end{definition}
\begin{proposition}
Let $\cF$ be the   class of \emph{all} real-valued functions from $\cX$ to $\bR$. For any $L >0$, we have $\vc(\cF^{0/1}(L/2\gamma)) = \fatsh_{\gamma}(\cF(L))$. Moreover, for any unlabeled set $U$, we have $\vc(\cF^{0/1}(L/2\gamma),U) = \fatsh_{\gamma}(\cF(L),U)$, where $\fatsh_{\gamma}(\cF(L),U) = \fatsh_{\gamma}(\cF(L)_{|U})$. 
\end{proposition}
\begin{proof}
    Since $\cF(L)$ is the class of all $L$-Lipschitz functions, we know from Lemma~2 in \citet{gottlieb2014efficient} that $\gamma$-fat shattering of $\cF(L)$ does not change if we only consider $\forall x, r(x) = 0$ as a witness of shattering. In this case we know that if a set is $\gamma$-shattered, any pair of points in the set have distance at least $2\gamma/L$. Thus, it can be shattered by $\cF^{0/1}(L/2\gamma)$. On the other hand, for any set that is shattered by $\cF^{0/1}(L/2\gamma)$, we know that every pair is $2\gamma/L$ apart and therefore can be $\gamma$-shattered by $\cF(L)$ too. 
\end{proof}
We now apply Theorem~\ref{thm:similarity_distance} by noting that $\cF^{0/1}(L/2\gamma)$ is a translated version of $\cF^{0/1}$ that only allows the forbidden behaviours with penalty at most $L/2\gamma$, i.e., pairs closer than $2\gamma/L$ must be labeled similarly. Hence, the growth rate of $\{\cF^{0/1}(L/2\gamma): \gamma>0\}$ is bounded by $O(m^2)$.
\begin{theorem}
    Let $\cF$ be a class of functions from $\cX$ to $\bR$ with Lipschitz constant at most $L$. There exists a learner $\cA: (\cX \times \{0,1\})^* \times \cX \rightarrow \{0,1\}$, with the following property: for every distribution $\cD$, every 
    $\delta \in (0,1)$ and $m \in \bN$ 
     we have with probability at least $1-\delta$ over samples $S \sim \cD^m$ that
       \[ \begin{aligned}
 \err(\cA(S),\cD) 
   & \leq \min_{\gamma>0} \left\{\err(\cF^{0/1}(L/2\gamma),\cD) 
    + O\left(\sqrt{\frac{\vc(\cF^{0/1}(L/2\gamma)) \log^2 (m) + \log (1/\delta)}{m}}\right)\right\} \\
    &\leq \min_{\gamma > 0 } \left\{\err^{\gamma}(\cF,\cD) 
    + O\left(\sqrt{\frac{\vc(\cF^{0/1}(L/2\gamma)) \log^2 (m) + \log (1/\delta)}{m}}\right)\right\}.
    \end{aligned}\]

Furthermore, if $\cF$ is the class of \emph{all} functions from $\cX$ to $\bR$ with Lipschitz constant at most $L$, $\diam(\cX) \leq R$, and $\ddim(\cX) \leq d$, then we have 
 \[ \begin{aligned}
 \err(\cA(S),\cD) 
    &\leq \min_{\gamma > 0 } \left\{\err^{\gamma}(\cF,\cD) 
    + O\left(\sqrt{\frac{\fatsh_{\gamma}(\cF) \log^2 (m) + \log (1/\delta)}{m}}\right)\right\}\\
    &\leq \min_{\gamma > 0 } \left\{\err^{\gamma}(\cF,\cD) 
    + O\left(\sqrt{\frac{(LR/\gamma)^{d+1}  \log^2 (m) + \log (1/\delta)}{m}}\right)\right\},
    \end{aligned}\]
where the last line follows from the fact that $\gamma$-fat shattering dimension of the class of all $L$ Lipschitz functions from $\cX$ to $R$ is bounded by $(LR/\gamma)^{d+1}$.
    \end{theorem}
Interestingly, since the above is proven using Theorem~\ref{thm:similarity_distance} and optimizing over $\cF^{0/1}(1/r)$ for $r\geq 0$, we can further compete with the minimum of the bound over different Lipschitz constants. In other words, instead of requiring the concept class to only consist of Lipschitz functions with some constant, we can take any family $\cF$ of real-valued functions. We then consider different subsets $\cF(L/2\gamma)$ for $L>0$ to create a collection $\mathbb{F} = \{\cF(L/2\gamma): \gamma,L > 0\}$ with $\tau_{\mathbb{F}}(m) = O(m^2)$, and then minimize the upper bound over $\cF(L/2\gamma)$ for $\gamma,L>0$. Note that these classes are embedded in each other, i.e., $\cF(L/2\gamma)\subseteq \cF(L'/2\gamma')$ for $L \leq L'$ and $\gamma \geq \gamma'$.
\begin{theorem}
    Let $\cF$ be a class of functions from $\cX$ to $\bR$. There exists a learner $\cA: (\cX \times \{0,1\})^* \times \cX \rightarrow \{0,1\}$, with the following property: for every distribution $\cD$, every 
    $\delta \in (0,1)$ and $m \in \bN$ 
     we have with probability at least $1-\delta$ over samples $S \sim \cD^m$ that
       \[ \begin{aligned}
 \err(\cA(S),\cD) 
   & \leq \min_{\gamma,
L>0} \left\{\err(\cF^{0/1}(L/2\gamma),\cD)
    + O\left(\sqrt{\frac{\vc(\cF^{0/1}(L/2\gamma)) \log^2 (m) + \log (1/\delta)}{m}}\right)\right\} \\
    &\leq \min_{\gamma,L > 0 } \left\{\err^{\gamma}(\cF(L),\cD) 
    + O\left(\sqrt{\frac{\vc(\cF^{0/1}(L/2\gamma)) \log^2 (m) + \log (1/\delta)}{m}}\right)\right\},
    \end{aligned}\]
where $\cF(L) \subseteq \cF$ is the set of all functions in $\cF$ with Lipschitz constant at most $L$. Furthermore, if $\cF$ is the class of \emph{all} functions from $\cX$ to $\bR$, $\diam(\cX) \leq R$, and $\ddim(\cX) \leq d$,  we have with probability at least $1-\delta$ over samples $S \sim \cD^m$ that
 \[ \begin{aligned}
 \err(\cA(S),\cD) 
    &\leq \min_{\gamma,L > 0 } \left\{\err^{\gamma}(\cF(L),\cD) 
    + O\left(\sqrt{\frac{\fatsh_{\gamma}(\cF(L)) \log^2 (m) + \log (1/\delta)}{m}}\right)\right\}\\
    &\leq \min_{\gamma,L > 0 } \left\{\err^{\gamma}(\cF(L),\cD) 
    + O\left(\sqrt{\frac{(LR/\gamma)^{d+1} \log^2 (m) + \log (1/\delta)}{m}}\right)\right\}.
    \end{aligned}\]
   \end{theorem}

It is useful to compare the above two theorems with the following bound in \citet{bartlett2002sample}, which serves the same purpose of optimizing the value of $\gamma$ in the bound. It is derived by applying a stratification on values of $\gamma$ and taking a union bound over the guarantees that hold for learning with respect to a fixed $\gamma$. Such guarantees for fixed $\gamma$ are derived by controlling the generalization error with the covering number and then bounding the covering number based on $\gamma$-fat shattering dimension \citep{bartlett2002sample,alon1997scale}. Although it is not possible to exactly compare $\fatsh_{\gamma}(\cF)$ with $\vc(\cF^{0/1}(2\gamma/L))$, we can still compare some specific cases. The following bound has an extra $O(\ln(1/\gamma)/m)$  dependence on the parameter $\gamma$, which makes the bound grow more rapidly with the decrease in $\gamma$. Particularly, the above results show that whenever the class consists of Lipschitz functions we can significantly improve the bounds. Moreover, in this setting, the following bound depends on $\fatsh_{\gamma/32}(\cF)$ which scales with an additional factor of $(32/\gamma)^d$, while the bound in the above theorem only scales with $(1/\gamma)^d$, saving an exponential factor of $32^d$.

\begin{theorem}[Corollary~9 in \cite{bartlett2002sample}]
    Let $\cF$ be a class of real-valued functions from $\cX$ to $\bR$. For every distribution $\cD$ and $m \in \bN$, we have with probability at least $1-\delta$ over $S \sim \cD^m$ that $ \forall f\in \cF, \forall \gamma\in (0,1]$
    \[
  \err(h,\cD) \leq \err^{\gamma}(h,S) + \sqrt{\frac{2}{m}\left(\fatsh_{\gamma/32}(\cF)\ln\left(\frac{34em}{\fatsh_{\gamma/32}(\cF)}\right)\log_2(578m) + \ln(\frac{8}{\gamma\delta})\right)}.
    \]
\end{theorem}

\subsection{Smoothness and Sparse Margin Assumptions} In this section we discuss how smoothness and sparsness of the labeling rule can be incorporated into learning. For a data generating distribution $\cD$, let $\eta_{\cD}(x) = \cD_{Y|X}[Y=1|X = x]$. We define 
\[\phi_{\cD} (\gamma) = \probs{x \sim \cD_{\cX}}{|\eta_{\cD}(x)-1/2|\leq \gamma} \]
to measure the sparseness of the probabilistic labeling rule inside a margin around $1/2$. For the metric space $(\cX,\rho)$, we also define 
\[\psi_{\cD}(L) =  \probs{x \sim \cD}{\probs{z \sim \cD_{\cX}}{|\eta_{\cD}(x) -\eta_{\cD}(z)| > L\rho(x,z)} > 0} \]
to measure how probabilistically far is the labeling rule from being $L$-Lipschitz. It is worth noticing that whenever the labeling rule is deterministic, i.e., $\eta_{\cD}(.) \in\{0,1\}$, the function $\psi_{\cD}(L)$ is just the pairwise probabilistic Lipschitzness function at parameter $L$ \citep{urner2013probabilistic,pentina2018multi}.

The following claims that it is possible to benefit from such conditions, even if the knowledge of the smoothness and the Lipschitz constant of the probabilistic labeling rule is not given. On the other hand, the guarantee is comparable with that of ERM up to additional factor of $O(\log(m)/\sqrt{m})$. Therefore, it offers a general paradigm with no dependence on the parameters of the underlying distribution while benefiting from its niceness, if it happens. Particularly, it offers a trade-off between the reduced complexity for restricted versions of the hypothesis class, i.e., $\vc(\cH(L/2\gamma))$, and the additional error rate due to $\phi_{\cD}(\gamma)$ and $\psi_{\cD}(L)$. The following can be proved  by noting that for $\mathbb{F} = \{\cF(L/2\gamma): L,\gamma >0\}$ we have $\tau_{\mathbb{F}} = O(m^2)$ and that for any $\gamma,L >0$, with high probability over $S \sim \cD^m$ there exists a function in $\cF(L/2\gamma)$ that agrees with $f^*_{\cD}$ on almost $1-\phi_{\cD}(\gamma) - \psi_{\cD}(L)$ fraction of $S$.
\begin{theorem}\label{thm:smooth_sparse_labeling}
    Let $\cH$ be a class of functions from $\cX$ to $\{0,1\}$. There exists a learner $\cA: (\cX \times \{0,1\})^* \times \cX \rightarrow \{0,1\}$, with the following property: for every distribution $\cD$ such that $f^*_{\cD} \in \cH$, every 
    $\delta \in (0,1)$ and $m \in \bN$ 
     we have with probability at least $1-\delta$ over samples $S \sim \cD^m$ that
       \[ \begin{aligned}
    &\err(\cA(S),\cD)
    \leq \min_{\gamma,L> 0 } \left\{\err(f^*_{\cD},\cD) 
    + \phi_{\cD} (\gamma)+ \psi_{\cD}(L) + O\left(\sqrt{\frac{\vc(\cH(L/2\gamma)) \log^2 (m) + \log (1/\delta)}{m}}\right)\right\},
    \end{aligned}\]
where 
$\allowdisplaybreaks \psi_{\cD}(L) =  \probs{x,z \sim \cD_{\cX^2}} { |\eta_{\cD}(x) -\eta_{\cD}(z)| > L\rho(x,z)}$ is smoothness of the labeling rule $\eta_{\cD}(.)$ and $\phi_{\cD} (\gamma) = \probs{x \sim \cD_{\cX}}{|\eta(x;\cD)-1/2|\leq \gamma} $ is its sparseness.
    \end{theorem}

It is useful to compare Theorem~\ref{thm:smooth_sparse_labeling} with the following result that is achievable by learning from the class of all real-valued $L$-Lipschitz functions using an ERM learner that minimizes the $\gamma$-error. 
\begin{theorem}
 Let $\cF$ be a class of functions from $\cX$ to $\bR$ with Lipschitz constant at most $L$ and assume $\diam(\cX) \leq R$, and $\ddim(\cX) \leq d$. For $\gamma > 0$, we have with probability at least $1-\delta$ over samples $S \sim \cD^m$ that
 \[ \begin{aligned}
 \err(\cA_{\gamma-\erm}(S),\cD) 
    \leq  \err(f^*_{\cD},\cD) 
    + \phi_{\cD} (\gamma)+ \psi_{\cD}(L) +  O\left(\sqrt{\frac{(LR/\gamma)^{d+1} \log^2 (m) + \log (1/\delta)}{m}}\right),
    \end{aligned}\]
     where $\cA_{\gamma-\erm}(S)$ chooses any hypothesis in $\cF$ with minimum $\err^{\gamma}(S)$ and uses its sign for prediction.
\end{theorem} 
The proof of above again follows from the fact that the fat-shattering dimension of $\cF$ is bounded by $(LR/\gamma)^{d+1}$ and that there exists a $L$-Lipschitz function that agrees with $f^*_{\cD}$ except on almost a $\phi_{\cD}(\gamma) + \psi_{\cD}(L)$ fraction of data. Theorem~\ref{thm:smooth_sparse_labeling} improves the above in two fronts. First, it applies to any arbitrary class $\cH$ and therefore, $\vc(\cH(L/2\gamma))$ can be significantly smaller than $(LR/\gamma)^{d+1}$.
Second, similar to learning real-valued functions, we get the advantage of relaxing the knowledge of $\gamma$ in advance and alternatively finding the minimum trade-off, among all values of $\gamma$, between the values of $\phi(\gamma)$, $\psi(\gamma)$, and the complexity term.

 {\bf Tsybakov Noise Condition.} The Tsybakov noise class $\tyb(a,\alpha)$ is a common noise condition studied in the literature \citep{mammen1999smooth,tsybakov2004optimal}. It is the class of all distributions $\cD$ such that (i) the Bayes classifier $f^*_{\cD}$ is in the hypothesis class reflecting prior knowledge and (ii) for all $\gamma > 0$ we have 
\[
\phi_{\cD}(\gamma) = \probs{x \sim \cD_{\cX}}{|\eta_{\cD}(x)-1/2|\leq \gamma} \leq a'\gamma^{\alpha/(1-\alpha)},
\]
where $a' = (1-\alpha)(2\alpha)^{\alpha/(1-\alpha)}a^{1/(1-\alpha)}$. This condition is a special case of the general setting discussed above. We have the following result for this class based on Theorem~\ref{thm:smooth_sparse_labeling}.
\begin{theorem}
    Let $\cH$ be a class of functions from $\cX$ to $\{0,1\}$. There exists a learner $\cA: (\cX \times \{0,1\})^* \times \cX \rightarrow \{0,1\}$, with the following property: for every $a,\alpha>0$ and every distribution $\cD\in \tyb(a,\alpha)$, every 
    $\delta \in (0,1)$ and $m \in \bN$ 
     we have with probability at least $1-\delta$ over samples $S \sim \cD^m$ that
       \[ \begin{aligned}
 \err(\cA(S),\cD) 
    &\leq \min_{\gamma> 0 } \left\{\err(f^*_{\cD},\cD) 
    + a'\gamma^{\alpha/(1-\alpha)} + O\left(\sqrt{\frac{\vc(\cH(L/2\gamma)) \log^2 (m) + \log (1/\delta)}{m}}\right)\right\},
    \end{aligned}\]
where $L$ is the Lispchitz constant of the probabilistic labeling rule $\eta_{\cD}(.)$.
    \end{theorem}

We want to highlight that while it is well-studied that Tsybakov noise condition is helpful in improving the dependency of the error rate on the sample size \citep{mammen1999smooth,tsybakov2004optimal,bartlett2006convexity}, the above result is offering a trade-off that considers how this condition reduces the complexity of learning by considering subsets $\cH(L/2\gamma) \subseteq \cH$. 
\subsection{Incorporating Prior Knowledge with Contrastive Assumption}

We consider a setting that has been studied as an approach for contrastive learning \cite{saunshi2019theoretical, awasthi2022more, alon2024optimal}. It is assumed that there exists a mapping under which points that are closer to each other tend to be `more similar' in their labels. For instance, this mapping can be a result of pre-training.  
The prior literature mostly focuses on learning this representation and then finding a linear classifier for prediction of the mapped data. Here, we take a different approach and show that such information can be used by any hypothesis class as part of the prior knowledge. 

We move beyond the strict assumption that for any $x$, if $x,x^+$ have opposing labels then any $x^-$ with $\phi(x,x^-) >\phi(x,x^-) $ must also have an opposite label to $x$. We suggest that this assumption holds only for triplets that lie inside a ball of radius at most $r$. We then consider the family of concept classes based on this assumption and apply Theorem~\ref{thm:family_forbidden_sample} to optimize the distance $r$. We believe such a result is helpful because an algorithm that was previously designed to find the mapping as a pre-training stage is no longer required to make sure the assumption holds for every triplet in the domain. It may happen that there are mappings that only satisfy the assumption on balls where the distribution is supported and can conclude a smaller generalization bound.

\begin{definition}[$\cB_{\cn},p_{\cn,\phi}$]
Let $x,x^{+},x^{-}\in\cX$ be any triplet of points such that $\phi(x,x^+)\leq \phi(x,x^-)$. The forbidden behaviours of contrastive assumption is defined as $\cB_{\nn}(x,x^+,x^-)= \{ (0,1,0),(1,0,1) \}$. The penalty $p_{\cn,\phi}(x,x^+,x^-)$ is $1/\rad_{\phi}{(x,x^+,x^-)}$, where $\rad_{\phi}{(x,x^+,x^-)}$ is the minimum radius of a ball that contains $x$,$x^+$.
\end{definition}

The class of functions $\cH_{p_{\cn,\phi}}(r)$ based on the above forbidden behaviours will then become 
  \[\begin{aligned} \cH_{p_{\cn,\phi}}(1/r) = \{h &\in \{0,1,\star\}^{\cX} : \exists h'\in\cH, \forall x\in \cX, \text{$h(x) = \star$ or $h(x) = h'(x)$} \\
& \text{ and }\forall (x,x^+,x^-) \in \cX^3 \text{ with $\rad_{\phi}(x,x^+,x^-) \leq r$, $\phi(x,x^+) < \phi(x,x^-)$,} \\
&\text{  and $h(x),h(x^+),h(x^-)\in \{0,1\}$ if $h(x) = h(x^-) $ then $h(x)=h(x^+)$}\}.\end{aligned}\] 

Motivated by the above, we define shattering at radius $r$, where at most two points can be shattered inside a ball of radius $r$. 
\begin{definition}[$\vc^{\rad}_{\phi,r}(\cH)$] Let $\cH\subseteq \{0,1\}^{\cX}$ be a hypothesis class. For any $r \geq 0$, a set $S\subseteq \cX$
is called to be shattered at radius $r$ if (i) $S$ is shattered by $\cH$ and (ii) for any ball $B\subseteq \cX$ of radius $r$, we have $|S \cap B| \leq 2$. The $\vc^{\rad}_{\phi,r}(\cH)$ dimension of the hypothesis class $\cH$ at radius $r$ is then defined as the size of the largest shattered set at distance $r$.
\end{definition}
Noting that $\tau_{\bH_{\cB_{\cn},\phi}}(m) = O(m^3)$, we can apply Theorem~\ref{thm:family_forbidden_sample} to the forbidden behaviours that respect the contrastive assumption. This yields an error bound that optimally, among $r\geq 0$, incorporates the contrastive assumption with the hypothesis class.
\begin{theorem}
    Let $\cH\subseteq \{0,1\}^{\cX}$ be a hypothesis class. There exists a learner $\cA$, with the following property: for every distribution $\cD$, every 
    $\delta \in (0,1)$ and $m \in \bN$ 
    we have with probability at least $1-\delta$ over samples $S \sim \cD^m$ that
     \[ \begin{aligned}
    &\err(\cA(S),\cD) \leq \min_{r \geq 0}\left\{\err(\cH(1/r),\cD) + O\left(\sqrt{\frac{\left(\vc_{\phi,r}^{\rad}(\cH)\right)\log^2 (m) + \log (1/\delta)}{m}}\right)\right\}.
   \end{aligned} \]
\end{theorem}

\subsubsection{Contrastive approach with no prior knowledge about a concept class.}

We know show that if we set the class to be the class of functions, then we would recover bounds similar to the bounds derived for nearest neighbour approach.

\begin{proposition}
     Let $\cH$ be the class of all functions from $\cX$ to $\{0,1\}$. Then, $|N_{2r}| \leq \vc_{\phi,r}^{\rad}(\cH) \leq 2|N_r|$, where $N_r$ denotes the $r$-net of domain $\cX$.
\end{proposition}
\begin{proof}
    Take a maximal set $U$ that is shattered and an $r$-net $N$ of the domain $\cX$. Map every point $x\in U$ to an element in $a\in N$ such that $\phi(x,a)\leq r$. This must exists because $N$ is an $r$-cover. Now for any $a\in N$, the inverse of the mapping must contain at most $2$ points from $U$, otherwise, three points have been mapped to $a$ which means they are all in $B(a,r)$. This concludes that $|U| \leq 2|N_r|$. Take any maximum $2r$-packing $N'$, which is indeed a $2r$-net. This set is obviously shattered because every $a,b\in N'$ has $\phi(a,b) > 2r$ and cannot lie inside a ball of radius $r$.   
\end{proof}

   \begin{remark}
        Regarding the relation between the two notions, we have that $\vc(\cH_{p_{\nn,\phi}}(1/2r)) \leq \vc(\cH_{p_{\cn,\phi}}(1/r)) \leq 2 \vc(\cH_{p_{\nn,\phi}}(1/r))$ and $\err(\cH_{p_{\cn,\phi}}(1/r),S) \leq \err(\cH_{p_{\nn,\phi}} (1/2r),S)$ on any sample $S$.
    \end{remark}
    
\begin{theorem}
    Let $\cH$ be the class of all functions from $\cX$ to $\{0,1\}$. Assume $\diam(\cX) \leq R$ and $\ddim(\cX) \leq d$. There exists a learner $\cA: (\cX \times \{0,1\})^* \times \cX \rightarrow \{0,1\}$, with the following property: for every distribution $\cD$, every 
    $\delta \in (0,1)$ and $m \in \bN$
     we have with probability at least $1-\delta$ over samples $S \sim \cD^m$ that
       \[ \begin{aligned}
 \err(\cA(S),\cD) 
   & \leq \min_{r\geq 0 } \left\{\err(\cH(1/r),\cD) 
    + O\left(\sqrt{\frac{2|N_{r}| \log^2 (m) + \log (1/\delta)}{m}}\right)\right\} \\
    &\leq \min_{r\geq 0 } \left\{\err(\cH(1/r),\cD) 
    + O\left(\sqrt{\frac{\left((R/r)^{d+1} \right) \log^2 (m) + \log (1/\delta)}{m}}\right)\right\}.
    \end{aligned}\]
    \end{theorem}

  \printbibliography

@article{shao2022theory,
  title={A theory of PAC learnability under transformation invariances},
  author={Shao, Han and Montasser, Omar and Blum, Avrim},
  journal={Advances in Neural Information Processing Systems},
  volume={35},
  pages={13989--14001},
  year={2022}
}

@inproceedings{saunshi2019theoretical,
  title={A theoretical analysis of contrastive unsupervised representation learning},
  author={Saunshi, Nikunj and Plevrakis, Orestis and Arora, Sanjeev and Khodak, Mikhail and Khandeparkar, Hrishikesh},
  booktitle={International Conference on Machine Learning},
  pages={5628--5637},
  year={2019},
  organization={PMLR}
}

@inproceedings{awasthi2022more,
  title={Do more negative samples necessarily hurt in contrastive learning?},
  author={Awasthi, Pranjal and Dikkala, Nishanth and Kamath, Pritish},
  booktitle={International conference on machine learning},
  pages={1101--1116},
  year={2022},
  organization={PMLR}
}

@inproceedings{alon2022theory,
  title={A theory of PAC learnability of partial concept classes},
  author={Alon, Noga and Hanneke, Steve and Holzman, Ron and Moran, Shay},
  booktitle={2021 IEEE 62nd Annual Symposium on Foundations of Computer Science (FOCS)},
  pages={658--671},
  year={2022},
  organization={IEEE}
}

@article{haussler1994predicting,
  title={Predicting $\{$0, 1$\}$-functions on randomly drawn points},
  author={Haussler, David and Littlestone, Nick and Warmuth, Manfred K},
  journal={Information and Computation},
  volume={115},
  number={2},
  pages={248--292},
  year={1994},
  publisher={Elsevier}
}

@article{haussler1995sphere,
  title={Sphere packing numbers for subsets of the Boolean n-cube with bounded Vapnik-Chervonenkis dimension},
  author={Haussler, David},
  journal={Journal of Combinatorial Theory, Series A},
  volume={69},
  number={2},
  pages={217--232},
  year={1995},
  publisher={Elsevier}
}

@article{kontorovich2017nearest,
  title={Nearest-neighbor sample compression: Efficiency, consistency, infinite dimensions},
  author={Kontorovich, Aryeh and Sabato, Sivan and Weiss, Roi},
  journal={Advances in Neural Information Processing Systems},
  volume={30},
  year={2017}
}

@inproceedings{hanneke2021stable,
  title={Stable sample compression schemes: New applications and an optimal SVM margin bound},
  author={Hanneke, Steve and Kontorovich, Aryeh},
  booktitle={Algorithmic Learning Theory},
  pages={697--721},
  year={2021},
  organization={PMLR}
}

@article{gottlieb2014near,
  title={Near-optimal sample compression for nearest neighbors},
  author={Gottlieb, Lee-Ad and Kontorovich, Aryeh and Nisnevitch, Pinhas},
  journal={Advances in Neural Information Processing Systems},
  volume={27},
  year={2014}
}

@article{gottlieb2014efficient,
  title={Efficient classification for metric data},
  author={Gottlieb, Lee-Ad and Kontorovich, Aryeh and Krauthgamer, Robert},
  journal={IEEE Transactions on Information Theory},
  volume={60},
  number={9},
  pages={5750--5759},
  year={2014},
  publisher={IEEE}
}

@inproceedings{
alon2024optimal,
title={Optimal Sample Complexity of Contrastive Learning},
author={Noga Alon and Dmitrii Avdiukhin and Dor Elboim and Orr Fischer and Grigory Yaroslavtsev},
booktitle={The Twelfth International Conference on Learning Representations},
year={2024}
}

@article{shawe2002structural,
  title={Structural risk minimization over data-dependent hierarchies},
  author={Shawe-Taylor, John and Bartlett, Peter L and Williamson, Robert C and Anthony, Martin},
  journal={IEEE transactions on Information Theory},
  volume={44},
  number={5},
  pages={1926--1940},
  year={2002},
  publisher={IEEE}
}

@article{bartlett2002sample,
  title={The sample complexity of pattern classification with neural networks: the size of the weights is more important than the size of the network},
  author={Bartlett, Peter L},
  journal={IEEE transactions on Information Theory},
  volume={44},
  number={2},
  pages={525--536},
  year={2002},
  publisher={IEEE}
}

@inproceedings{urner2013probabilistic,
  title={Probabilistic lipschitzness a niceness assumption for deterministic labels},
  author={Urner, Ruth and Ben-David, Shai},
  booktitle={Learning Faster from Easy Data-Workshop@ NIPS},
  volume={2},
  pages={1},
  year={2013}
}

@inproceedings{pentina2018multi,
  title={Multi-task $\{$K$\}$ ernel $\{$L$\}$ earning Based on $\{$P$\}$ robabilistic $\{$L$\}$ ipschitzness},
  author={Pentina, Anastasia and Ben-David, Shai},
  booktitle={Algorithmic Learning Theory},
  pages={682--701},
  year={2018},
  organization={PMLR}
}

@article{david2016supervised,
  title={Supervised learning through the lens of compression},
  author={David, Ofir and Moran, Shay and Yehudayoff, Amir},
  journal={Advances in Neural Information Processing Systems},
  volume={29},
  year={2016}
}

@article{boucheron2005theory,
  title={Theory of classification: A survey of some recent advances},
  author={Boucheron, St{\'e}phane and Bousquet, Olivier and Lugosi, G{\'a}bor},
  journal={ESAIM: probability and statistics},
  volume={9},
  pages={323--375},
  year={2005},
  publisher={EDP Sciences}
}

@article{schapire2013boosting,
  title={Boosting: Foundations and algorithms},
  author={Schapire, Robert E and Freund, Yoav},
  journal={Kybernetes},
  volume={42},
  number={1},
  pages={164--166},
  year={2013},
  publisher={Emerald Group Publishing Limited}
}

@book{shalev2014understanding,
  title={Understanding machine learning: From theory to algorithms},
  author={Shalev-Shwartz, Shai and Ben-David, Shai},
  year={2014},
  publisher={Cambridge university press}
}

@inproceedings{montasser2019vc,
  title={Vc classes are adversarially robustly learnable, but only improperly},
  author={Montasser, Omar and Hanneke, Steve and Srebro, Nathan},
  booktitle={Conference on Learning Theory},
  pages={2512--2530},
  year={2019},
  organization={PMLR}
}

@article{attias2022characterization,
  title={A characterization of semi-supervised adversarially robust pac learnability},
  author={Attias, Idan and Hanneke, Steve and Mansour, Yishay},
  journal={Advances in Neural Information Processing Systems},
  volume={35},
  pages={23646--23659},
  year={2022}
}

@article{koltchinskii2002rademacher,
  title={Rademacher penalties and structural risk minimization},
  author={Koltchinskii, Vladimir},
  journal={IEEE Transactions on Information Theory},
  volume={47},
  number={5},
  pages={1902--1914},
  year={2002},
  publisher={IEEE}
}

@article{bartlett2002model,
  title={Model selection and error estimation},
  author={Bartlett, Peter L and Boucheron, St{\'e}phane and Lugosi, G{\'a}bor},
  journal={Machine Learning},
  volume={48},
  number={1},
  pages={85--113},
  year={2002},
  publisher={Springer}
}

@article{tsybakov2004optimal,
  title={Optimal aggregation of classifiers in statistical learning},
  author={Tsybakov, Alexander B},
  journal={The Annals of Statistics},
  volume={32},
  number={1},
  pages={135--166},
  year={2004},
  publisher={Institute of Mathematical Statistics}
}

@article{bartlett2006convexity,
  title={Convexity, classification, and risk bounds},
  author={Bartlett, Peter L and Jordan, Michael I and McAuliffe, Jon D},
  journal={Journal of the American Statistical Association},
  volume={101},
  number={473},
  pages={138--156},
  year={2006},
  publisher={Taylor \& Francis}
}

@article{mammen1999smooth,
  title={Smooth discrimination analysis},
  author={Mammen, Enno and Tsybakov, Alexandre B},
  journal={The Annals of Statistics},
  volume={27},
  number={6},
  pages={1808--1829},
  year={1999},
  publisher={Institute of Mathematical Statistics}
}

@inproceedings{kearns1990efficient,
  title={Efficient distribution-free learning of probabilistic concepts},
  author={Kearns, Michael J and Schapire, Robert E},
  booktitle={Proceedings [1990] 31st Annual Symposium on Foundations of Computer Science},
  pages={382--391},
  year={1990},
  organization={IEEE}
}

@inproceedings{pollard1990empirical,
  title={Empirical processes: theory and applications},
  author={Pollard, David},
  year={1990},
  organization={Ims}
}

@article{alon1997scale,
  title={Scale-sensitive dimensions, uniform convergence, and learnability},
  author={Alon, Noga and Ben-David, Shai and Cesa-Bianchi, Nicolo and Haussler, David},
  journal={Journal of the ACM (JACM)},
  volume={44},
  number={4},
  pages={615--631},
  year={1997},
  publisher={ACM New York, NY, USA}
}

\appendix
\section{Missing Proofs from Section~\ref{sec:learning_family_concepts}}\label{appendixB}

We prove our generalization bounds using a compression-based argument. We will employ the one-inclusion graph predictor and a boosting algorithm to prove the existence of compression sets. 
The one-inclusion predictor $\cA_{\cH}$ for the hypothesis class $\cH$, introduced by \cite{haussler1994predicting} will predict as follows. Given a labeled set  $S=\{(x_1,y_1),\ldots,(x_n,y_n)\}\in (\cX\times \{0,1\})^n$ and a test point $x_{n+1}$, $\cA(S)$ will create a graph with nodes for any (total) $h \in \cH_{|\dom(S) \cup x_{n+1}}$ and edges that connect two neighbouring nodes, i.e., two nodes $h$ and $h'$ that differ in exactly one of $x_i\in \dom(S)\cup x_{n+1}$ and are similar otherwise. It will then find an orientation of the edges that minimizes the maximum out-degree of the nodes. Then, if there exists two hypotheses $h$ and $h'$ consistent with $S$, it will predict $\hat{y}_{n+1}$ based on the orientation of the edge connecting the nodes $h$ and $h'$. If no node is consistent, we set $\cA(S)(x_{n+1}) = 0$.
\citet{haussler1994predicting} prove that for any total concept class $\cH$ there exists an orientation with out-degree at most $\vc(\cH)$ and that the one-inclusion graph predictor can be used as a learner achieving $\frac{\vc(\cH)}{n+1}$ leave-one-out error. \citet{alon2022theory} show that essentially the same result holds for any partial concept class and by requiring $\cA_{\cH}$ to only consider total behaviours on its nodes.

\begin{lemma}[Theorem~2.3 of \cite{haussler1994predicting}, Lemma~35 of \cite{alon2022theory}]\label{lemma:oig_symmetric}
    For any hypothesis class $\cH$ with VC dimension $\vc(\cH)$ there is a function $\cA_{\cH}:(\cX\times \{0,1\})^* \times \cX \rightarrow \{0,1\}$ such that for any $n \in \bN$ and any sample $S=\{(x_1,y_1),\ldots,(x_n,y_n)\}\in (\cX\times \{0,1\})^n$ realizable by $\cH$ we have
    \[
    \frac{1}{n!}\sum_{\sigma \in \Gamma}\indicator{\cA\left(x_{\sigma(1)},y_{\sigma(1)},\ldots,x_{\sigma(n-1)},y_{\sigma(n-1)},x_{\sigma(n)}\right) \neq y_{\sigma(n)}} \leq \frac{\vc(\cH)}{n},
    \]
    where $\Gamma$ denotes the symmetric group on $[n]$.
\end{lemma}

 A compression scheme is a pair $(\rho,\kappa)$ where $\kappa: (\cX \times \{0,1\})^{*} \rightarrow (\cX \times \{0,1\})^*$ is a compression map and $\rho: (\cX \times \{0,1\})^{*} \rightarrow \{0,1\}^{\cX}$ is a decompression map. For a set $S$, the size of the compression is defined by $|\kappa(S)|$ and the reconstruction error will be defined by $\err(\rho(\kappa(S)),S)$.

It is possible to show the existence of a compression scheme by a well-known technique of boosting the error of the the one-inclusion graph learners through majority voting. This boosting technique is discussed in Lemma~\ref{lemma:boosting} and is used in many works to prove the existence of compression schemes, \cite{david2016supervised,alon2022theory}, etc. This could also be extended to find a compression scheme with small reconstruction error using the reduction to realizable technique \cite{david2016supervised}.

We first show a bound on
$\tau_{\bH}(.)$ lets us reliably estimate, for all $\cH \in \bH$, the generalization error of mappings that are based on aggregations of OIG learners $\cA_{\cH}$ on subsets of sample. A similar argument is used to prove the generalization error of AdaBoost through hybrid compression schemes, that is, decompression functions that can be chosen from a class, but for the case of realizable compression schemes \cite{schapire2013boosting}. We prove an agnostic type guarantee and show that decompression error of schemes associated with every $\cH\in\bH$ is a good estimate of their expected error. Informally, we show that the behaviour of the schemes $\{\bA_{\cH}(.),\cH \in \bH\}$, as defined in the follwing, on the reconstructed part can be bounded by $\tau_{\bH}(.)$.

For a set $S = ((x_1,y_1),\ldots, (x_m,y_m))$ and a set of indices $I\subseteq [n]$, we denote $S_{I} := \left((x_{i},y_i)\right)_{i \in I}$. Moreover for any two integers $a\leq b$ we denote $S_{a:b} := \left((x_{i},y_i)\right)_{a \leq i \leq b}$
\begin{lemma}[Decompression schemes for $\bH$]\label{lemma:decomp_H}
    Let $\bH$ be a collection of hypothesis classes.  Let $k,T\in\bN$ and denote $\bA_{\cH}(S_1,\ldots,S_T)(.) :=  \text{Majority}(\cA_{\cH}(S_1)(.),\ldots, \cA_{\cH}(S_{T})(.))$.
    For any $m\in\bN$, $\delta \in (0,1)$, and any distribution $\cD$ over $\cX \times \{0,1\}$, denoting $B(\bH,k,T,m) = (\log(\tau_{\bH}(m)) + kT)\log(m)$, with probability at least $1-\delta$ over $S \sim \cD^m$ for all $S_1,\ldots,S_T \in  \{(x_1,y_1),\ldots,(x_m,y_m)\}^k$ and $\cH \in \bH$, that 
\[
\begin{aligned}
   & |\err(\bA_{\cH}(S_1,\ldots,S_T), \cD) - \err(\bA_{\cH}(S_1,\ldots,S_T), S)| \leq \\
    & c\sqrt{\err(\bA_{\cH,I}(S), V_I)\frac{B(\bH,k,T,m) + \log(\frac{1}{\delta})}{m}} + c \frac{B(\bH,k,T,m) + \log(\frac{1}{\delta}) + k}{m}\\
    & \leq c'\sqrt{\frac{(\log(\tau_{\bH}(m)) + kT)\log(m) + \log(\frac{1}{\delta})}{m}}
\end{aligned}
\]
for some constants $c$ and $c'$.
\end{lemma}
\begin{proof}
    For a fixed $S_1,\ldots, S_T$, let $V := S \setminus \bigcup S_i$. We know that the size of $[\cH \sim_{S}]_{\bH}$  is bounded by  $\tau_{\bH}(m)$. Note that for any $\cH$ and $\cH'$ such that $\cH \sim_{S} \cH'$, we have that $\cA_{\cH}(S_i)(x) = \cA_{\cH'}(S_i)(x)$ for all $x \in V$ and $i \in [T]$. This is because $\cH$ and $\cH'$ will induce the exact same (total) restriction on any subset of $S$, including $S_i \bigcup x$. Therefore, having a fixed mapping of graphs to orientations, we can observe that the two OIGs will predict the same on $V$. Thus, the total number of behaviours that $\{\bA_{\cH}(S_1,\ldots,S_T): \cH \in \bH\}$ induces on $V$ is bounded by the number of equivalence classes on $S$, which is at most $\tau_{\bH}(m)$. The above implies that if we fix the set of indices $I = (i_1,\ldots,i_{kT}) \in [m]^{kT}$ before observing the sample and consider the class $\cF_I = \{\bA_{\cH,I}(.): \cH \in \bH\}$, that on any $S \sim \cD^m$, runs the majority vote of OIGs on the indices in $I$, i.e. $\bA_{\cH,I}(.) = \text{Majority}(\cA_{\cH}(S_{1:k})(.),\ldots, \cA_{\cH}(S_{k(T-1):kT})(.) )$,  then the instances in $V_I:= S_{[n]\setminus I}$ will remain i.i.d. and also $\vc(\cF_I,V_I) \leq \log(\tau_{\bH}(m))$. Therefore, letting $\delta' = \frac{\delta}{m^{kT}}$, we apply the uniform convergence property of $\cF_I$ to conclude that for some constant $c$, with probability at least $1-\delta'$ over $S \sim \cD^m$ we have for all $\cH \in \bH$,
    \[
    \begin{aligned}
         & |\err(\bA_{\cH,I}(S), \cD) - \err(\bA_{\cH,I}(S), V_I)|        \leq \\
         & c\sqrt{\err(\bA_{\cH,I}(S), V_I)\frac{B(\bH,k,T,m)+ \log(\frac{1}{\delta})}{m-kT}} + c \frac{B(\bH,k,T,m) + \log(\frac{1}{\delta})}{m-kT}.
    \end{aligned}
    \]
    Now a union bound over all $m^{kT}$ sequence of indices $I \in [m]^{kT}$ implies that with probability at least $1-\delta$ over $S \sim \cD^m$ we have for all $I \in [m]^{kT}$ and $\cH \in \bH$,
  \[
    \begin{aligned}
         & |\err(\bA_{\cH,I}(S), \cD) - \err(\bA_{\cH,I}(S), V_I)|        \leq \\
         & c\sqrt{\err(\bA_{\cH,I}(S), V_I)\frac{B(\bH,k,T,m) + \log(\frac{1}{\delta})}{m-kT}} + c \frac{B(\bH,k,T,m) + \log(\frac{1}{\delta})}{m-kT}.
    \end{aligned}
    \]
    Note that in the above we have used a version of uniform convergence that is derived by Bernstein's inequality, see Corollary~5.2 in \cite{boucheron2005theory}. We will need this Bernstein version to be able to connect bound with high probability the deviation between $\err(\bA_{\cH,I}(S), \cD)$ and $\err(\bA_{\cH,I}(S), S)$. The hybrid compression argument in \cite{schapire2013boosting} uses the standard bounds used to prove the uniform convergence for VC classes that does not have the multiplicative factor of $\err(\bA_{\cH,I}(S), V_I)$ in the bound. This is because they only discuss realizable setting and when the validation error is exactly zero. 
    
    Now we will connect the above bound on the difference between $\err(\bA_{\cH,I}(S), \cD)$ and $\err(\bA_{\cH,I}(S), V_I)$ to bound the difference between $\err(\bA_{\cH,I}(S), \cD)$ and  $\err(\bA_{\cH,I}(S), S)$. This is discussed in Lemma~A.1 in \cite{david2016supervised} for a single predictor $h$, that is, a single decompression of a compressed subset. We prove this again for completeness and because we want the property to hold for the decompressions associated with every $\cH \in \bH$ . We know that $|\err(\bA_{\cH,I}(S), V_I) - \err(\bA_{\cH,I}(S), S)|\leq kT/m$  and, thus, $|\err(\bA_{\cH,I}(S), \cD) - \err(\bA_{\cH,I}(S), V_I)| \geq |\err(\bA_{\cH,I}(S), \cD) - \err(\bA_{\cH,I}(S), S)| -kT/m$. Since $kT \leq m/2$, for any fixed $\cH \in \bH$, we can conclude the event that 
  \begin{equation}\label{eq:event1}
    \begin{aligned}
         & |\err(\bA_{\cH,I}(S), \cD) - \err(\bA_{\cH,I}(S), V_I)|        \geq \\
         & c\sqrt{\err(\bA_{\cH,I}(S), V_I)\frac{B(\bH,k,T,m) + \log(1/\delta)}{m-kT}} + c \frac{B(\bH,k,T,m) + \log(1/\delta)}{m-kT}.
    \end{aligned}
    \end{equation}
    is implied by the event
     \begin{equation}\label{eq:event2}
             \begin{aligned}
         & |\err(\bA_{\cH,I}(S), \cD) - \err(\bA_{\cH,I}(S), S)|        \geq \\
         & c\sqrt{\err(\bA_{\cH,I}(S), V_I)\frac{B(\bH,k,T,m) + \log(\frac{1}{\delta})}{m}} + c \frac{B(\bH,k,T,m) + \log(\frac{1}{\delta}) + k}{m}.
    \end{aligned}
         \end{equation}
    Therefore, the event $E_1$ that there exists some $\cH \in \bH$ satisfying \ref{eq:event1} is implied by the event $E_2$ that there exists $\cH \in \bH$ satisfying \ref{eq:event2}. Since the first event $E_1$ happens with probability at most $\delta$, the event $E_2$ will happen with probability at most $\delta$ too. Hence, with probability at least $1-\delta$ over $S \sim \cD^m$ we have for all $I \in [m]^{kT}$ and $\cH \in \bH$,
    \[
    \begin{aligned}
         & |\err(\bA_{\cH,I}(S), \cD) - \err(\bA_{\cH,I}(S), S)|        \leq \\
         & c\sqrt{\err(\bA_{\cH,I}(S), V_I)\frac{B(\bH,k,T,m) + \log(\frac{1}{\delta})}{m}} + c \frac{B(\bH,k,T,m) + \log(\frac{1}{\delta}) + k}{m},
    \end{aligned}
    \]
    as desired.
\end{proof}

We now introduce the guarantee that we can get by boosting weak-learners which we will use to prove the next theorem.

\begin{lemma}[Boosting]\label{lemma:boosting}
    Let $k,m\in\bN$ and $S = \{(x_1,y_1),\ldots,(x_m,y_m)\} \in (\cX \times \{0,1\})^m$ be a set. Assume there exists an algorithm  $\cA:(\cX\times \{0,1\})^k \times \cX \rightarrow \{0,1\}$ such that for every distribution $\cD$ on $\cX \times \{0,1\}$ with $\cD[S] = 1$, there exists $S_{\cD} \in \{(x_1,y_1),\ldots,(x_m,y_m)\}^k$ such that $\err(\cA(S_{\cD}),\cD) \leq 1/3$. Then there exists a constant $c \geq 1$ such that there exists $S_1,\ldots,S_T \in  \{(x_1,y_1),\ldots,(x_m,y_m)\}^k$ with $T = \lceil c\log m \rceil$ such that for the function $\allowdisplaybreaks \hat{h}(.) = \text{Majority} (\cA(S_1)(.) , \ldots, \cA(S_T)(.))$ we have $\hat{h}(x_i) = y_i$ for all $i \in [m]$.
\end{lemma}

{\bf{Proof of Theorem~\ref{thm:learning_family_concepts}.}}
Fix $\cH \in \bH$ and let $\delta' = \delta/2$. We know with probability at least $1-\delta'$ over $S\sim \cD^m$ that there exists $\hat{h} \in \cH$ with $\err(\hat{h},S) \leq \err(\cH,\cD) + O(\sqrt{\frac{\log(1/\delta')}{m}})$ by e.g., a Hoeffding's inequality. 
Let $S'$ be the longest realizable subsequence of $S$ and note that $S' \geq (1- \err(\hat{h},S))|S|$. 
We now show that there exists $S_1,\ldots,S_T$ for $T = \lceil c''\log(m) \rceil$ for some constant $c''$ such that $\bA_\cH(S_1,\ldots,S_T)(x) := \text{Majority}(\cA_{S_1}(x),\ldots, \cA_{S_T}(x))$ is equal to $y$ for any $(x,y) \in S'$.

Observe that for every distribution $\cP$ over $\cX \times \{0,1\}$ that is realizable with $\cH$, we have for $\cA_{\cH}$ that 
\[
\begin{aligned}
   &\expects{R \sim \cP^n,(x,y)\sim \cP}{1 \{\cA_{\cH}(R)(x) \neq y\}} \\
   &= \frac{1}{(n+1)!}\sum_{\sigma \in \Gamma}\expects{R \sim \cP^{n+1}}{\indicator{\cA\left(x_{\sigma(1)},y_{\sigma(1)},\ldots,x_{\sigma(n-1)},y_{\sigma(n-1)}\right)(x_{\sigma(n)}) \neq y_{\sigma(n)}} }\\
   &= \expects{R \sim \cP^{n+1}}{\frac{1}{(n+1)!}\sum_{\sigma \in \Gamma}\indicator{\cA\left(x_{\sigma(1)},y_{\sigma(1)},\ldots,x_{\sigma(n-1)},y_{\sigma(n-1)}\right)(x_{\sigma(n)}) \neq y_{\sigma(n)}}}\\
   &\leq \frac{ \expects{R \sim \cP^{n+1}}{\vc(\cH,R)}}{n+1}, 
\end{aligned}
\]
where the last line is due to Lemma~\ref{lemma:oig_symmetric}. Particularly, Lemma~2 in \citet{haussler1995sphere} proves by a shifting argument that we can bound the density of the OIG by $\vc(\cH,R)$. The proof of Lemma~\ref{lemma:oig_symmetric} is based on relating the error rate to the density of OIG (see \citet{haussler1994predicting}), which gives the above inequality. Now, for any distribution $\cP'$ over $S'$ we know that this distribution is realizable by $\cH$ and therefore, $\expects{R \sim \cP^n,(x,y)\sim \cP'}{1 \{\cA_{\cH}(R)(x) \neq y\}} \leq \frac{\vc(\cH,R)}{n+1} \leq \frac{\vc(\cH)}{n+1}$. Now let $n = 3 \vc(\cH)$ so that $\expects{R \sim \cP^n,(x,y)\sim \cP'}{1 \{\cA_{\cH}(R)(x) \neq y\}} \leq 1/3$. This proves the existence of a set $R' \in S^k$ of size $k: = 3 \vc(\cH)$ with $\err(\cA_{\cH}(R'),\cP') \leq 1/3$. This is sufficient to invoke Lemma~\ref{lemma:boosting} and conclude that the claimed $S_1,\ldots,S_T$ exists. Therefore, we have  $S_1,\ldots,S_T$ such that $\err(\bA_\cH(S_1,\ldots,S_T) , S) \leq \err(\hat{h},S)$. Let $I^*$ denote the sequence of the indices of the sets $S_1,\ldots,S_T$ in the same order.

We now describe the learner $\cA_{\bH}$. The learner will go over all sets $J\in [m]^{i}$ for any $i \in [m]$. For any $J$ it runs $\bA_{\cH,J}(S)$ for all $\cH \in \bH$. It then evaluates $\err(\bA_{\cH,J}(S), S)$ and returns a predictor $\cA_{\bH}(S)(.): = \bA_{\bar{\cH},\bar{J}}(S)(.)$ as follows:
\[
\cA_{\bH}(S)(.) \in \arg\min_{\substack{J\in [m]^{i}, i\in [m]\\ \cH \in \bH}} \err(\bA_{\cH,J}(S), S) +  c'\sqrt{\frac{(\log(\tau_{\bH}(m)) + |J|)\log(m) + \log(1/\delta)}{m}}
\]

Set $\delta'' = \delta'/m$. For any fixed $k' \leq \lceil m/T \rceil$ we know from Lemma~\ref{lemma:decomp_H} that with probability at least $1-\delta''$ over $S\sim \cD^m$ for all $\cH \in \bH$ and all $J \in [m]^{k'T}$ we have 
\[
  \err(\bA_{\cH,J}(S), \cD) \leq \err(\bA_{\cH,J}(S), S) +  c'\sqrt{\frac{(\log(\tau_{\bH}(m)) + k'T + 1)\log(m) + \log(1/\delta')}{m}}
\]
Taking a union bound over all values of $k' \leq \lceil m/T \rceil$ implies that with probability at least $1-\delta'$ over $S\sim \cD^m$ for all $\cH \in \bH$, all $k' \leq \lceil m/T \rceil$ and $J \in [m]^{k'T}$ we have 
\begin{equation}\label{eq:event3}
  \err(\bA_{\cH,J}(S), \cD) \leq \err(\bA_{\cH,J}(S), S) +  c'\sqrt{\frac{(\log(\tau_{\bH}(m)) + k'T + 1)\log(m) + \log(1/\delta')}{m}}.
\end{equation}
We proved that there exists sets $S_1,\ldots,S_T$ such that $\err(\bA_{\cH,I^*}(S),S) \leq \err(\hat{h},S)$ and hence for the returned predictor $\bA_{\bar{\cH},\bar{J}}(S)(.)$ we have 
\[
  \err(\bA_{\bar{\cH},\bar{J}}(S), S) \leq \err(\hat{h},S) +  c'\sqrt{\frac{(\log(\tau_{\bH}(m)) + k'T + 1)\log(m) + \log(1/\delta')}{m}}.
\]
Taking a union bound over the events that $\err(\hat{h},S)$ is a good estimate of $\err(\cH,\cD)$, and the event in Equation~\ref{eq:event3}, we get with probability at least $1-\delta$ over $S \sim \cD^m$ that 
\[
\begin{aligned}
    &\err(\cA_{\bH}(S),\cD) \\
    &\leq  \err(\bA_{\bar{\cH},\bar{J}}(S), S) + c'\sqrt{\frac{(\log(\tau_{\bH}(m)) + |\bar{J}|+ 1)\log(m) + \log(2/\delta)}{m}} \\
    & \leq \err(\bA_{\cH,I^*}(S), S) + c'\sqrt{\frac{(\log(\tau_{\bH}(m)) + kT+ 1)\log(m) + \log(2/\delta)}{m}} \\  
     & \leq \err(\hat{h}, S) + c'\sqrt{\frac{(\log(\tau_{\bH}(m)) + \vc(\cH))c''\log(m)+ 1)\log(m) + \log(2/\delta)}{m}} \\  
     & \leq \err(\cH, \cD) + c'\sqrt{\frac{(\log(\tau_{\bH}(m)) + \vc(\cH))c''\log(m)+ 1)\log(m) + \log(2/\delta)}{m}}  + \sqrt{\frac{\log(2/\delta)}{m}}\\
     & \leq \err(\cH,\cD)  + O\left(\sqrt{\frac{(\log(\tau_{\bH}(m)) + \vc(\cH))\log^2(m) + \log(1/\delta)}{m}}\right).
\end{aligned}
\]

Since the choice of $\cH \in \bH$ was arbitrary, we can start with setting it to be $\cH^* \in \bH$ that achieves the minimum value of
\[\min_{\cH \in \bH}\left\{\err(\cH, \cD) + c'\sqrt{\frac{(\log(\tau_{\bH}(m)) + \vc(\cH))c''\log(m)+ 1)\log(m) + \log(2/\delta)}{m}}  + \sqrt{\frac{\log(2/\delta)}{m}}\right\},
\]
and argue that the predictor chosen by $\cA_{\bH}$ competes with the error of $\cH^*$. This concludes the proof. \hfill\qed

{\bf  Proof of Theorem~\ref{thm:family_concepts_data}.}
    The proof is similar to the proof of Theorem~\ref{thm:learning_family_concepts}. Let $\delta' = \delta/4$. We only need to note two properties. First, argue we can find sets $S_1,\ldots,S_T$ with $|S_i| \leq 3\vc(\cH,\cD,m,\delta')$  such that $\bA_{\cH}(S)(.): = \text{Majority}(\cA_{\cH}(S_1)(.),\ldots, \cA_{\cH}(S_T)(.))$ has $\err(\bA_{\bH}(S),S) \leq \err(\hat{h},S)$ as defined in the proof of Theorem~\ref{thm:learning_family_concepts}. Observe that we proved the existence of $S_i$ by showing that $\expects{R \sim \cP^n,(x,y)\sim \cP'}{1 \{\cA_{\cH}(R)(x) \neq y\}} \leq \frac{\vc(\cH,R)}{n+1} \leq \frac{\vc(\cH)}{n+1} $. We can now simply argue that by definition we have with probability at least $1 - \delta'$ over $S \sim \cD^{m}$ that $\vc(\cH,S) \leq \vc(\cH,\cD,m,\delta')$ and therefore $\expects{R \sim \cP^n,(x,y)\sim \cP'}{1 \{\cA_{\cH}(R)(x) \neq y\}}  \leq \frac{\vc(\cH,\cD,m,\delta')}{n+1}$. This proves that we can find $S_i$'s with $|S_i| \leq 3\vc(\cH,\cD,m,\delta')$ that satisfy what we desired, i.e., $\cA_{\cH}(S_i)$ is a weak learner for the distributions over sample that we wanted.
    Second, we need to note that for a set of indices $I\in [m]^{kT}$ the uniform convergence property for $\cF_{I}$ as defined in Lemma~\ref{lemma:decomp_H} is satisfied with an error that depends on $\tau_{\bH}(m,\cD,\delta')$. The way we argued that $|\err(\bA_{\cH,I}(S),\cD) - \err(\bA_{\cH,I}(S),V_I)|$ is small was by noting that the function $\bA_{\cH,I}(S)$ is only defined using $S \setminus V_I$ and instances in $V_I$ are still i.i.d. from $\cD$. We then invoked the uniform convergence for $\cF_{I}$ on $S$. Now notice that because of the monotonicity of $\tau_{\bH}(.)$ we have that $\probs{S \sim \cD^m}{\tau_{\bH}(S) \leq \tau_{\bH}(\cD,m,\delta')} \leq \probs{ V\sim \cD^{m-kT}}{\tau_{\bH}(V) \leq \tau_{\bH}(\cD,m,\delta')} $. Therefore, we conclude that with probability at least $1-\delta'$ over $S \sim \cD^m$ we have $\vc(\cF_{I}, V_I) \leq \log(\tau_{\bH}(\cD,m,\delta'))$. Hence, we get the desired concentration of errors of $\bA_{\cH}$ on $V_I$ to expected error for all $\cH \in \bH$ with the same rate but by replacing $\tau_{\bH}(m)$ with $\tau_{\bH}(\cD,m,\delta')$. Taking a union bound over the events that $\err(\hat{h},S)$ is a good estimate of $\err(\cH,\cD)$, that $\vc(\cH,S) \leq \vc(\cH,\cD,m,\delta')$, that $\tau_{\bH}(S) \leq \tau_{\bH}(\cD,m,\delta')$, the event in Equation~\ref{eq:event3}, we get with probability at least $1-\delta$ over $S \sim \cD^m$ that 
    \[
\begin{aligned}
    \err(\cA_{\bH}(S),\cD)  \leq \err(\cH,\cD)  + O\left(\sqrt{\frac{(\log(\tau_{\bH}(m,\cD,\delta/4)) + \vc(\cH,\cD,m,\delta/4))\log^2(m) + \log(1/\delta)}{m}}\right).
\end{aligned}
\]\hfill \qed
\subsection{Proof of Proposition~\ref{proposition:bad_srm}}

In this section, we prove that for every standard SRM learner there exists a distribution on which SRM has high error while the learner in Theorem~\ref{thm:learning_family_concepts} (or Theorem~\ref{thm:family_concepts_data}) achieves small generalization error on this distribution due to its data-dependent approach.

{\bf Proof of Proposition~\ref{proposition:bad_srm}.}
    We first define the collection $\bH$. For any $n\in \bN$, we define $\widetilde{\cH}_n = \{h \in \{0,1\}^{\bN}: \forall i \notin [n(n-1)+1, n^2],\,h(i) = 1\}$. We also define $h^*_n\in \{0,1\}^{\bN}$ as $h^*_n(i) = 0$ for all $i \in [n^2+1, n^2+\lceil n/2 \rceil]$ and $h^*_n(i) = 1$ otherwise.
    In other words, the class $\widetilde{\cH}_n $ induces every binary behaviour on $[n(n-1)+1, n^2]$ while every hypothesis in it is constant $1$ on every other natural number. The hypothesis $h^*_n$ is $1$ on every natural number except the interval $[n^2, n^2+\lceil n/2 \rceil]$. Define $\cH_n = \widetilde{\cH}_n \bigcup h^*_n$ and let $\bH = \{\cH_n: n\in \bN\}$. In words, we have consecutive pairs of intervals of increasing size and reserve the $n$th pair for $\cH_n$. Clearly, $\vc(\cH_n) = n$ and $\bH$ contains hypothesis classes of arbitrary large VC dimension. Let $\cA_{\srm}$ be any SRM learner that defines a weight function $w(n)$ to assign to each $\cH_n$ and returns a function $h \in \bigcup_{n\in \bN}\cH_n$ from the collection that minimizes \[
    \err(h,S) + \sqrt{\frac{\vc(\cH_{n(h)}) + \log(1/w(n(h))) +\log (1/\delta)}{m}},
    \]
    where $n(h)$ is the smallest index such that $h \in \cH_{n(h)}$.
    
   Take the hypothesis $h_1$ with $h_1(i) =1$ for all $i\in\bN$. Clearly, $h_1 \in \cH_1$.  Let $w_0 \in \bN$ be such that 
   \[
 \sqrt{w_0/m} > 1/2 + C_1\sqrt{\log (1/\delta)/m} +  C_2\sqrt{\frac{\log(1/w(1)) + \log (1/\delta)}{m}}, 
   \]
   where $C_1$ and $C_2$ are the constants for the Hoeffding's inequality and uniform convergence guarantee, respectively. Take the hypothesis class $\cH_{m_0}$ such that $\log(1/w(\cH_{m_0})) \geq w_0$. It is easy to observe that such class must exists because $\sum_{n\in \mathbb{N}}w(\cH_n)\leq 1$ and there cannot be any positive lower bound on the weight function. Take the function $h^*_{m_0}$ and observe that $h^*_{m_0}$ is only a member of $\cH_{m_0}$ and therefore $\vc(H_{n(h^*_{m_0})}) = \vc(\cH_{m_0}) = m_0$. Now, let $\cD$ be the uniform distribution on $[m_0^2+1, m_0^2 + m_0 + 1]$ that is realized by $h^*_{m_0}$, i.e., on the second interval for the $m_0$th pair. 
   Note that  $h_1$ is constant $1$ on the support of $\cD$ and we have with probability at least $1-\delta$ over $S \sim \cD^m$ that $\err(h_1,S) \leq 1/2 + C_1\sqrt{\log (1/\delta)/m}$. This means that the generalization term that $\cA_{\srm}$ considers for $h_1$ is  at most $1/2 + C_1\sqrt{\log (1/\delta)/m} + C_2\sqrt{\frac{1 + \log(1/w(1)) + \log (1/\delta)}{m}} $. Since we chose $m_0$ such that this term is less than $\sqrt{w_0/m}$, we conclude that $\cA_{\srm}$ will never return $h^*_{m_0}$. Note that we picked $\cH_{m_0}$ by making sure its weight is so small that the generalization bound $\cA_{\srm}$ considers is very large. Even if $\cA_{\srm}$ is careful in using $\vc(\cH_{n(h)},S)$ instead of $\vc(\cH_{n(h)})$, it will still not choose the function $h^*_{m_0}$ because although $\vc(\cH_{m_0},S)$ is small, the term $\sqrt{\log(1/w(\cH_{m_0})/m}$ is large. Now observe that for any $h\in \bigcup_{n\in\bN} \cH_n$ with $h \neq h^*_{m_0}$ we have $\err(h,\cD) = 1/2$. This proves that with probability at least $1-\delta$ over $S\sim \cD^m$ we have $\err(\cA_{\srm}(S),\cD) = 1/2$. On the other hand, every $\cH_n$ with $n\neq m_0$ induces the same behaviour on any $S \sim \cD^m$, i.e., the constant $1$ behaviour. Therefore, $\tau_{\bH}(m,\cD,\delta) = 2$. Moreover, $\vc(\cH_{m_0},\cD) = 0$. This implies that for the learner $\cA$ in Theorem~\ref{thm:learning_family_concepts} (or Theorem~\ref{thm:family_concepts_data}) we have with probability at least $1-\delta$ that  $\err(\cA(S),\cD) \leq O\left(\sqrt{\frac{\log^2(m) + \log(1/\delta)}{m}}\right)$, which concludes the proof.\hfill \qed

\end{document}